\documentclass{article}

\usepackage{arxiv}

\usepackage[utf8]{inputenc} 
\usepackage[T1]{fontenc}    
\usepackage{hyperref}       
\usepackage{url}            
\usepackage{booktabs}       
\usepackage{amsfonts}       
\usepackage{nicefrac}       
\usepackage{microtype}      
\usepackage{amsmath}
\usepackage{amsfonts}
\usepackage{amsthm}
\usepackage{amssymb}
\usepackage{enumitem}
\usepackage{graphicx}
\usepackage{comment}
\usepackage{thmtools} 
\usepackage{thm-restate}

\usepackage[backend=biber,style=numeric,maxcitenames=2,maxbibnames=8,minbibnames=4,uniquelist=false,uniquename=init,isbn=false,
doi=false,url=false,firstinits,sortcites]{biblatex}
\let\citet\textcite
\let\citep\parencite
\usepackage[ruled,vlined]{algorithm2e}
\SetKwInput{KwParam}{Parameters}
\usepackage[capitalize,noabbrev]{cleveref} 
\usepackage{xcolor}

\newcommand{\vx}{\mathbf{x}}
\newcommand{\vs}{\mathbf{s}}
\newcommand{\vxt}{\vx_t}
\newcommand{\vy}{\mathbf{y}}

\newcommand{\vf}{\mathbf{f}}
\newcommand{\vM}{\mathbf{M}}
\newcommand{\vh}{\mathbf{h}}
\newcommand{\vzeta}{\boldsymbol{\zeta}}

\newcommand{\vstm}{\vs_{t-1}}

\newcommand{\mA}{\mathbf{A}}
\newcommand{\mI}{\mathbf{I}}

\newcommand{\sR}{\mathbb{R}}

\newcommand{\xhatm}{\hat{\vx}}
\newcommand{\xhatpm}{\hat{\vx}'}
\newcommand{\yhatm}{\hat{\vy}}
\newcommand{\yhatpm}{\hat{\vy}'}
\DeclareMathOperator{\MMD}{MMD}

\newcommand{\What}{\widehat{\mathcal{W}}}
\newcommand{\Wcemix}{\What_{c,\eps}^{\,\textrm{mix}}}

\newcommand{\iidsim}{\stackrel{\text{i.i.d}}{\sim}}
\newcommand{\httpsurl}[1]{\href{https://#1}{\nolinkurl{#1}}}

\newtheorem{thm}{Theorem}[section]

\newtheorem{lem}[thm]{Lemma}
\newtheorem{prop}[thm]{Proposition}

\newtheorem{defn}[thm]{Definition}

\newtheorem{rem}[thm]{Remark}
\newtheorem{exa}[thm]{Example}
\numberwithin{equation}{section}
\numberwithin{thm}{section}

\newcommand{\eps}{\varepsilon}

\newcommand{\E}{\mathbb{E}}
\newcommand{\N}{\mathbb{N}}

\newcommand{\R}{\mathbb{R}}

\newcommand{\Ccal}{\mathcal{C}}

\newcommand{\Fcal}{\mathcal{F}}
\newcommand{\Hcal}{\mathcal{H}}
\newcommand{\Kcal}{\mathcal{K}}
\newcommand{\Lcal}{\mathcal{L}}
\newcommand{\Mcal}{\mathcal{M}}
\newcommand{\Pcal}{\mathcal{P}}

\newcommand{\Xcal}{\mathcal{X}}
\newcommand{\Ycal}{\mathcal{Y}}
\newcommand{\Zcal}{\mathcal{Z}}
\newcommand{\Wcal}{\mathcal{W}}

\DeclareMathOperator{\argmax}{argmax}

\newcommand{\mytitle}{COT-GAN: Generating Sequential Data\\ via Causal Optimal Transport}

\title{\mytitle}
\author{%
  Tianlin Xu \\
  London School of Economics\\
  \texttt{t.xu12@lse.ac.uk}
  \And
  Li K. Wenliang \\
  University College London\\
  \texttt{kevinli@gatsby.ucl.ac.uk}
  \And
  Michael Munn \\
  Google, NY \\
  \texttt{munn@google.com}
  \And
  Beatrice Acciaio \\
  London School of Economics\\
  ETH Zurich\\
  \texttt{beatrice.acciaio@math.ethz.ch}
  
}

\begin{document}

\maketitle

\begin{abstract}
We introduce COT-GAN, an adversarial algorithm to train implicit generative models optimized for producing sequential data. The loss function of this algorithm is formulated using ideas from Causal Optimal Transport (COT), which combines classic optimal transport methods with an additional temporal causality constraint. Remarkably, we find that this causality condition provides a natural framework to parameterize the cost function that is learned by the discriminator as a robust (worst-case) distance, and an ideal mechanism for learning time dependent data distributions. Following Genevay et al.\ (2018), we also include an entropic penalization term which allows for the use of the Sinkhorn algorithm when computing the optimal transport cost. Our experiments show effectiveness and stability of COT-GAN when generating both low- and high-dimensional time series data. The success of the algorithm also relies on a new, improved version of the Sinkhorn divergence which demonstrates less bias in learning. 
\end{abstract}

\section{Introduction}

Dynamical data are ubiquitous in the world, including natural scenes such as video and audio data, and temporal recordings such as physiological and financial traces.
Being able to synthesize realistic dynamical data is a challenging unsupervised learning problem and has wide scientific and practical applications.
In recent years, training implicit generative models (IGMs) has proven to be a promising approach to data synthesis, 
driven by the work on generative adversarial networks (GANs) \citep{Goo}.

Nonetheless, training IGMs on dynamical data
poses an interesting yet difficult challenge. On 
one hand, learning complex spatial structures of static images has already received
significant effort within the research community. On the other hand, temporal dependencies are no less 
complicated since the dynamical features are strongly correlated with
spatial features. Recent works, including \citep{tgan,TimeGAN,wavegan,vgan,mocogan}, often tackle this problem by separating the model or loss into static and dynamic components. 

In this paper, we examine training dynamic IGMs for sequential data.  
We introduce a \textbf{new adversarial objective} that builds on
optimal transport (OT) theory, and constrains the 
transport plans to respect  \emph{causality}: 
the probability mass moved to the target sequence at time $t$ can only depend on the source sequence up to time $t$, see \citep{ABJ,BBLZ}.
A reformulation of the causality constraint 
leads to a new adversarial training objective, in the spirit of \cite{GPC} but tailored to sequential data. In addition, we demonstrate that optimizing the original Sinkhorn 
divergence over mini-batches causes biased parameter estimation, 
and propose the \textbf{mixed Sinkhorn divergence} which 
mitigates this problem. 
Our new framework, Causal Optimal Transport GAN (COT-GAN), outperforms
existing methods on a wide range
of datasets from traditional time series to high dimensional videos.

\section{Background}\label{section:introduction}
\subsection{Adversarial learning for implicit generative models}
\citet{Goo} introduced an adversarial scheme for training an IGM.
Given a (real) data distribution $\mu=\frac{1}{N}\sum_{i=1}^N\delta_{x^i},\,x^i\in\Xcal$, and a distribution 
$\zeta$ on some latent space $\Zcal$,
the generator is a function $g:\Zcal\to\Xcal$ trained so that the induced distribution $\nu=\zeta\circ g^{-1}$ is as close as possible to $\mu$ as 
judged by a discriminator. 
The discriminator is a function $f:\Xcal\to[0,1]$ trained to output a high value if the input is real (from $\mu$), and a low value otherwise (from $\nu$). 
In practice, the two functions are implemented as neural networks $g_\theta$ and $f_\varphi$ with parameters $\theta$ and $\varphi$, 
and the generator distribution is 
denoted by $\nu_\theta$.
The training objective is then formulated as a zero-sum game between the generator and the discriminator.
Different probability divergences were later proposed to evaluate the distance between $\mu$ and $\nu_\theta$ \citep{fGAN,mmdGAN,sobolevGAN,sMMDGan}.
Notably, the Wasserstein-1 distance was used in \citep{Arj,arjovsky2017towards}:
\begin{equation}\label{eq:W1_primal}
\mathcal{W}_1(\mu,\nu)=\inf_{\pi\in\Pi(\mu,\nu)}\E^\pi[\|x-y\|_1],
\end{equation}
where $\Pi(\mu,\nu)$ is the space of transport plans (couplings) between $\mu$ and $\nu$. 
Its dual form turns out to be a maximization problem over $\varphi$ such that $f_\varphi$ is Lipschitz. Combined with the minimization over $\theta$, 
a min-max problem can be formulated with a Lipschitz constraint on $f_\varphi$.

\subsection{Optimal transport and Sinkhorn divergences}\label{sect.GPC}
The optimization in \eqref{eq:W1_primal} 
is a special case of the classical (Kantorovich) optimal transport problem.
Given probability measures 
$\mu$ on $\Xcal$, $\nu$ on $\Ycal$, and a cost function $c:\Xcal\times\Ycal\to\R$, 
the optimal transport problem is formulated as
\begin{equation}\label{Wc_primal}
\Wcal_c(\mu,\nu):=\inf_{\pi\in\Pi(\mu,\nu)}\E^{\pi}[c(x,y)].
\end{equation}
Here, $c(x,y)$ represents the cost of transporting a unit of mass from $x\in\Xcal$ to $y\in\Ycal$, and $\Wcal_c(\mu,\nu)$ is thus the minimal total cost to transport the mass from $\mu$ to $\nu$. 
Obviously, the Wasserstein-1 distance \eqref{eq:W1_primal} corresponds to $c(x,y)=\|x-y\|_1$.
However, when $\mu$ and $\nu$ are supported on finite sets of size $n$, 
solving \eqref{Wc_primal} has super-cubic (in $n$) complexity \citep{cuturi2013sinkhorn, orlin1993faster,pele2009fast}, which is computationally expensive for 
large datasets.

Instead, \citet{GPC} proposed training IGMs by minimizing a regularized Wasserstein distance that can be computed more efficiently by the Sinkhorn algorithm; see \citep{cuturi2013sinkhorn}.
For transport plans with marginals $\mu$ supported on a finite set $\{x^i\}_i$ and $\nu$ on a finite set $\{y^{j}\}_j$, any $\pi\in\Pi(\mu,\nu)$ is also discrete with support on the set of all possible pairs $\{(x^i,y^{j})\}_{i,j}$. Denoting $\pi_{ij}=\pi(x^i,y^{j})$, the Shannon entropy of $\pi$ is given by
$
\textstyle{H(\pi):= -\sum_{i,j} \pi_{ij}\log(\pi_{ij})}.
$
For $\eps>0$, the regularized optimal transport problem then reads as
\begin{equation}\label{eq:regOT}
\Pcal_{c,\eps}(\mu,\nu):=\inf_{\pi\in\Pi(\mu,\nu)}\{\E^\pi[c(x,y)]-\eps H(\pi)\}.
\end{equation}
Denoting by $\pi_{c,\eps}(\mu,\nu)$ the optimizer in \eqref{eq:regOT}, one can define a regularized distance by
\begin{equation}\label{eq:Wce}
\mathcal{W}_{c,\eps}(\mu,\nu):= 
\E^{\pi_{c,\eps}(\mu,\nu)}[c(x,y)].
\end{equation}
Computing this distance is numerically more stable than solving the dual formulation of the OT problem, as the latter requires differentiating dual Kantorovich potentials; see e.g. \cite[Proposition~3]{bousquet2017optimal}.
To correct the fact that $\mathcal{W}_{c,\eps}(\alpha,\alpha)\neq0$, \citet{GPC} proposed to use the \emph{Sinkhorn divergence}
\begin{equation}\label{Sink}
\widehat{\mathcal{W}}_{c,\eps}(\mu,\nu):= 
2\mathcal{W}_{c,\eps}(\mu,\nu)-\mathcal{W}_{c,\eps}(\mu,\mu)-\mathcal{W}_{c,\eps}(\nu,\nu)
\end{equation}
as the objective function, and to learn the cost  $c_\varphi(x,y)=\|f_\varphi(x)-f_\varphi(y)\|$ parameterized by $\varphi$, resulting in the following adversarial objective
\begin{equation}
\label{minmaxOT}
\inf_\theta \sup_\varphi\, {\widehat{\mathcal{W}}}_{c_\varphi,\eps}(\mu,\nu_\theta).
\end{equation}
In practice, a sample-version of this cost is used, where $\mu$ and $\nu$ 
are replaced by distributions of mini-batches randomly extracted from them.

\section{Training generative models with Causal Optimal Transport}\label{section:background}
We now focus on data that consists of $d$-dimensional (number of channels), $T$-long sequences, so that $\mu$ and $\nu$ are distributions on the path space $\R^{d\times T}$. In this setting we introduce a special class of transport plans, between $\Xcal=\R^{d\times T}$ and $\Ycal=\R^{d\times T}$, that will be used to define our objective function; see Definition~\ref{def:causal-plan}. On $\Xcal\times\Ycal$, 
we denote by $x=(x_1,...,x_T)$ and $y=(y_1,...,y_T)$ the first and second half of the coordinates, and we 
let $\mathcal{F}^\Xcal=(\mathcal{F}^\Xcal_t)_{t=1}^T$ and $\mathcal{F}^\Ycal=(\mathcal{F}^\Ycal_t)_{t=1}^T$ be the canonical filtrations (for all $t$, $\mathcal{F}^\Xcal_t$ is the smallest $\sigma$-algebra s.t. $(x_1,...,x_T)\mapsto(x_1,...,x_t)$ is measurable; analogously for $\mathcal{F}^\Ycal$). 

\subsection{Causal Optimal Transport}\label{subsection:COT}
\begin{defn}\label{def:causal-plan}
A transport plan $\pi\in\Pi(\mu,\nu)$ is called causal if
\[
\pi(dy_t|dx_1,\cdots,dx_T)=\pi(dy_t|dx_1,\cdots,dx_t)\qquad \text{for all\, $t=1,\cdots,T-1$}.
\]
The set of all such plans will be denoted by $\Pi^{\Kcal}(\mu,\nu)$.
\end{defn}
Roughly speaking, the amount of mass transported by $\pi$ to a subset of the target space $\Ycal$ belonging to $\mathcal{F}^\Ycal_t$ depends on the source space $\Xcal$ only up to time $t$. 
Thus, a causal plan transports $\mu$ into $\nu$ in a non-anticipative way, which is a natural request in a sequential framework. 
In the present paper, we will use causality in the sense of Definition~\ref{def:causal-plan}. Note that, in the literature, the term causality is often used to indicate a mapping in which the output at a given time $t$ depends only on inputs up to time $t$.

Restricting the space of transport plans to $\Pi^\Kcal$ in the OT problem \eqref{Wc_primal} gives the COT problem  
\begin{equation}\label{COT}
\mathcal{K}_c(\mu, \nu) := \inf_{\pi \in \Pi^{\Kcal}(\mu, \nu)} \E^{\pi}[c(x,y)].
\end{equation}
COT has already found wide application in dynamic problems in stochastic calculus and mathematical finance, see e.g. \cite{ABZ,ABC,ABJ,BBBE,BBBW}.
The causality constraint can be equivalently formulated in several ways, see \cite[Proposition 2.3]{BBLZ}.  We recall here the formulation most well-suited for our purposes. Let
 $\Mcal(\Fcal^\Xcal,\mu)$ be the set of $(\Xcal,\Fcal^\Xcal,\mu)$-martingales, and define
\[
\Hcal(\mu):=\{(h,M) : h=(h_t)_{t=1}^{T-1},\ h_t\in\Ccal_b(\R^{d\times t}),\                       M=(M_t)_{t=1}^{T}\in\Mcal(\Fcal^\Xcal,\mu),\  M_t\in \Ccal_b(\R^{d\times t})\},
\]
where, as usual, $\Ccal_b(\mathbb{X})$ denotes the space of continuous, bounded functions on $\mathbb{X}$.
Then, a transport plan $\pi\in\Pi(\mu,\nu)$ is causal if and only if
\begin{equation}\label{causalhM}
\textstyle{\E^{\pi}\left[\sum_{t=1}^{T-1} h_t(y_{\leq t}) \Delta_{t+1}M(x_{\leq t+1})\right] = 0\;\; \text{for all $(h,M)\in\Hcal(\mu)$}},
\end{equation}
where $x_{\leq t} := (x_1, x_2, \dots, x_t)$ and similarly for $y_{\leq t}$, and $\Delta_{t+1}M(x_{\leq t+1}) := M_{t+1}(x_{\leq t+1}) - M_t(x_{\leq t})$. 
Therefore $\Hcal(\mu)$ acts as a class of test functions for causality. Intuitively, causality can be thought of as conditional independence (``given $x_{\leq t}$, $y_t$ is independent of $x_{>t}$''), that can be expressed in terms of conditional expectations. This in turn naturally lends itself to a formulation involving martingales.
Where no confusion can arise, with an abuse of notation we will simply write $h_t(y), M_t(x), \Delta_{t+1} M(x)$ rather than $h_t(y_{\leq t}), M_t(x_{\leq t}),\Delta_{t+1} M(x_{\leq t+1})$.

\subsection{Regularized Causal Optimal Transport}\label{sec:reg_cot}

In the same spirit of \cite{GPC}, we include an entropic regularization in the COT problem \eqref{COT} and consider
\begin{equation}\label{eq:reg_cot_problem}
\mathcal{P}^{\Kcal}_{c, \eps}(\mu, \nu) :=  \inf_{\pi \in \Pi^{\Kcal}(\mu, \nu)} \left\{\E^{\pi}[c(x,y)] - \eps H(\pi)\right\}.
\end{equation}
The solution to such problem is then unique due to strict concavity of $H$.
We denote by $\pi^{\Kcal}_{c,\eps}(\mu,\nu)$ the optimizer to the above problem, and define the regularized COT distance by
\[
\mathcal{K}_{c,\eps}(\mu,\nu):= 
\E^{\pi^{\Kcal}_{c,\eps}(\mu,\nu)}[c(x,y)].
\]

\begin{rem}\label{rem.lim}
In analogy to the non-causal case, it can be shown that, for discrete $\mu$ and $\nu$ (as in practice), the following limits holds:
\[
\mathcal{K}_{c}(\mu, \nu)\xleftarrow[\eps\to 0]{} \mathcal{K}_{c, \eps}(\mu, \nu) \xrightarrow[\eps\to \infty]{} \E^{\mu\otimes\nu}[c(x,y)],
\]
where $\mu\otimes\nu$ denotes the independent coupling.
\end{rem}
See \cref{sect.app.lim} for a proof.
This means that the regularized COT distance is between the COT distance and 
the loss obtained by independent coupling, and is closer to the former for small $\eps$.
Optimizing over the space of causal plans $\Pi^\Kcal(\mu,\nu)$ 
is not straightforward. Nonetheless, the following
proposition shows that the problem can be reformulated as a maximization over non-causal problems with respect to a specific family of cost functions.

\begin{prop}\label{cotsup}
The regularized COT problem \eqref{eq:reg_cot_problem} can be reformulated as
\begin{equation}\label{Pcce}
\Pcal^\Kcal_{c,\eps}(\mu,\nu) = \sup_{l\in\Lcal(\mu)} \mathcal{P}_{c+l, \eps}(\mu, \nu),
\end{equation}
where 
\begin{equation}\label{eq:L_set}
\Lcal(\mu):=\Bigg\{\sum_{j=1}^J \sum_{t=1}^{T-1} h^j_t(y)\Delta_{t+1}M^j(x): J\in\N, (h^j,M^j)\in\Hcal(\mu) \Bigg\}.
\end{equation}
\end{prop}
 This means that the optimal value of the regularized COT problem equals the maximum value over the family of regularized OT problems w.r.t. the set of cost functions $\{c+l : l\in\Lcal(\mu)\}$.
 This result has been proven in \cite{ABJ}. As it is crucial for our analysis, we show it in \cref{sec:cot_prob_proof}.

Proposition~\ref{cotsup} suggests the following worst-case distance 
between $\mu$ and $\nu$:
\begin{equation}\label{eq.rd}
\sup_{l\in\Lcal(\mu)} \mathcal{W}_{c+l, \eps}(\mu, \nu),
\end{equation}
as a regularized Sinkhorn distance that respects the causal constraint on the transport plans. 

In the context of training a dynamic IGM, the training dataset is a collection of paths 
$\{x^i\}_{i=1}^N$ of equal length $T$, $x^i=(x_1^i,..,x_T^i)$, $x^i_t\in\R^d$.
As $N$ is usually very large, we proceed as usual by approximating $\mathcal{W}_{c+l, \eps}(\mu, \nu)$ with
its empirical mini-batch counterpart. 
Precisely, for a given IGM $g_\theta$, we fix a batch size $m$ and sample $\{{x}^i\}_{i=1}^m$ 
from the dataset and $\{{z}^i\}_{i=1}^m$ from $\zeta$. Denote the generated samples by
${y}_\theta^i=g_\theta({z}^i)$, and the empirical distributions by
\[
\xhatm=\frac{1}{m}\sum_{i=1}^m\delta_{{x}^i},\quad
\yhatm_\theta=\frac{1}{m}\sum_{i=1}^m\delta_{y_\theta^i}.
\]
The empirical distance $\mathcal{W}_{c+l, \eps}(\xhatm, \yhatm_\theta)$ can be efficiently approximated 
by the Sinkhorn algorithm.

\begin{figure}
    \centering
    \includegraphics[width=\textwidth]{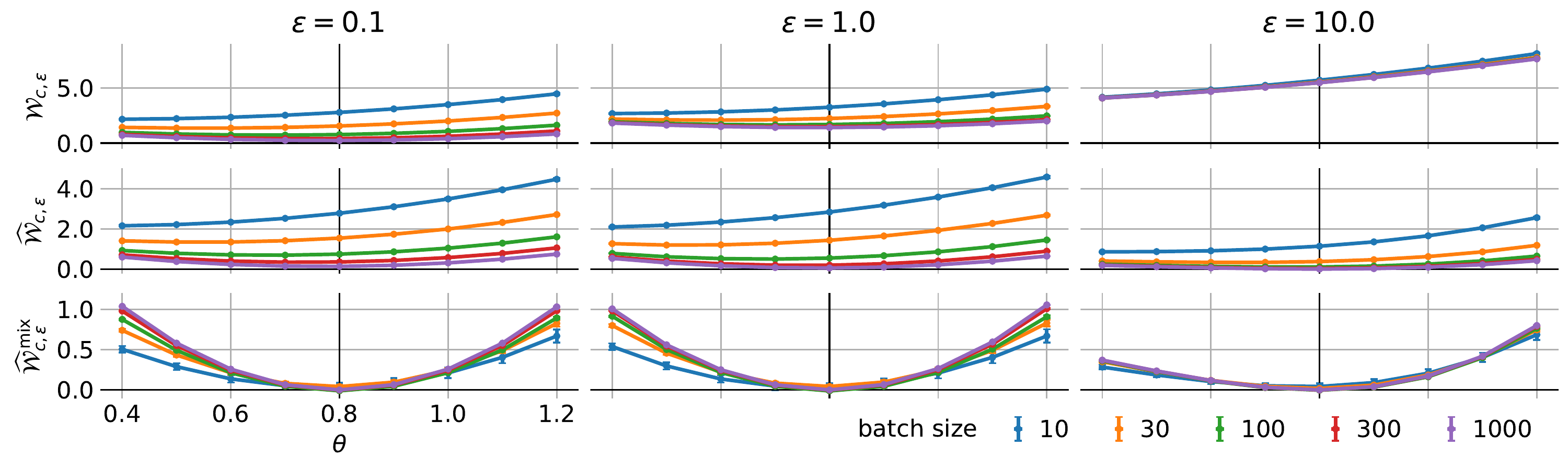}
    \caption{
    Regularized distance \eqref{eq:Wce}, Sinkhorn divergence \eqref{Sink}
    and mixed Sinkhorn divergence \eqref{Sink_W4}
   computed for mini-batches of size $m$ from $\mu$ and 
    $\nu_\theta$, where $\mu=\nu_{0.8}$. Color indicates batch size.
    Curve and errorbar show the mean and sem estimated from 300  
    draws of mini-batches.
    }
    \label{fig:sink_bias}
\end{figure}

\subsection{Reducing the bias with mixed Sinkhorn divergence}\label{sec:bias}
When implementing the Sinkhorn divergence \eqref{Sink} at the level of mini-batches, one canonical candidate clearly is
\begin{equation}\label{Sink_gpc}
2\mathcal{W}_{c_{\varphi}, \eps}(\xhatm,\yhatm_\theta)
-\mathcal{W}_{c_{\varphi}, \eps}(\xhatm, \xhatm)
-\mathcal{W}_{c_{\varphi}, \eps}(\yhatm_\theta, \yhatm_\theta),
\end{equation}
which is indeed what is used in \cite{GPC}.
While the expression in \eqref{Sink_gpc} does converge in expectation to \eqref{Sink} for $m\to\infty$ (\citep[Theorem 3]{genevay2019sample}), it is not clear whether 
it is an adequate loss given data of fixed batch size $m$.
In fact, we find that this is not the case, and demonstrate it here empirically.
\begin{exa}\label{bias_example}
We build an example where the data distribution $\mu$ belongs to a parameterized family of distributions $\{\nu_\theta\}_\theta$, with $\mu=\nu_{0.8}$ (details in \cref{sect.app.sink}).
As shown in \cref{fig:sink_bias} (top two rows), neither the expected regularized distance \eqref{eq:Wce} nor the Sinkhorn divergence \eqref{Sink} 
reaches minimum at $\theta=0.8$, especially for small $m$. 
This means that optimizing $\nu$ over mini-batches will not lead to $\mu$.
\end{exa}

Instead, we propose the following \emph{mixed Sinkhorn divergence} at the level of mini-batches:
\begin{equation}\label{Sink_W4}
\Wcemix(\xhatm, \xhatpm, \yhatm_\theta, \yhatpm_\theta)
:=\mathcal{W}_{c, \eps}(\xhatm, \yhatm_\theta)
+\mathcal{W}_{c, \eps}(\xhatpm, \yhatpm_\theta)
-\mathcal{W}_{c, \eps}(\xhatm, \xhatpm)
-\mathcal{W}_{c, \eps}(\yhatm_\theta, \yhatpm_\theta),
\end{equation}
where $\xhatm$ and $\xhatpm$ are the empirical distributions of mini-batches from the data distribution,
and $\yhatm_\theta$ and $\yhatpm_\theta$ from the IGM distribution $\zeta\circ g_\theta^{-1}$.
The idea is to take into account the bias within both the distribution $\mu$ as well as the distribution $\nu_\theta$ when sampling mini-batches.

Similar to  \eqref{Sink_gpc}, when the batch size $m\to\infty$, \eqref{Sink_W4} also converges to \eqref{Sink} in expectation. So, the natural question arises: for a fixed $m\in\N$, which of the two does a better job in translating the idea of the Sinkhorn divergence at the level of mini-batches? Our experiments suggest that \eqref{Sink_W4} is indeed the better choice. As shown in \cref{fig:sink_bias} (bottom row), $\Wcemix$ finds the correct minimizer for all $m$ in Example~\ref{bias_example}. 
To support this finding, note that the triangular inequality implies
\[
\E\left[ \left|
\mathcal{W}_{c_{\varphi}, \eps}(\hat{\textbf{x}}, \hat{\textbf{y}}_\theta)
+\mathcal{W}_{c_{\varphi}, \eps}(\hat{\textbf{x}}', \hat{\textbf{y}}_\theta')-
2\mathcal{W}_{c,\eps}(\mu,\nu)\right|\right]
\leq
2\E\left[\left|
\mathcal{W}_{c_{\varphi}, \eps}(\hat{\textbf{x}}, \hat{\textbf{y}}_\theta)-\mathcal{W}_{c,\eps}(\mu,\nu)\right|\right].
\]
One can possibly argue that in \eqref{Sink_W4} we are using two batches of size $m$, thus simply considering a larger mini-batch in \eqref{Sink_gpc}, say of size $2m$, may perform just as well. However, we found this not to be the case and our experiments confirm that the mixed Sinkhorn divergence \eqref{Sink_W4} does outperform \eqref{Sink_gpc} even when we allow for larger batch size. This reasoning can be extended by considering $\mathcal{W}_{c, \eps}(.,.)$ with more terms for different combinations of mini-batches. In fact, this is what is done in \citep{mixed_sink}, which came to our attention after submitting this paper for review.
We have tested different variations in several experiments and while empirically there is no absolute winner, adding more mini-batches increases the computational cost; see Appendix \ref{sect.app.sink}.
 
\subsection{COT-GAN: Adversarial learning for sequential data}
We now combine the results in \cref{sec:reg_cot} and \cref{sec:bias} to formulate an adversarial training algorithm for IGMs. 
First, 
we approximate the set of functions \eqref{eq:L_set} by truncating the sums at a fixed $J$, 
and we parameterize $\vh_{\varphi_1}:=(h_{\varphi_1}^j)_{j=1}^J$ and 
$\vM_{\varphi_2}:=(M_{\varphi_2}^j)_{j=1}^J$ as two separate 
neural networks, and let $\varphi:=(\varphi_1,\varphi_2)$.  To capture the adaptedness of those processes, we employ architectures where the output at time $t$ depends on the input only up to time $t$.
The mixed Sinkhorn divergence
between $\xhatm$ and $\yhatm_\theta$ is then 
calculated with respect to
a parameterized cost function
\begin{equation}\label{eq:cot_cost}
    c^\Kcal_{\varphi}(x,y):=c(x,y) + \sum_{j=1}^J\sum_{t=1}^{T-1} h^{j}_{\varphi_1,t}(y) \Delta_{t+1}M^{j}_{\varphi_2}(x).
\end{equation}

Second, it is not obvious how to directly impose the martingale condition, 
as constraints involving conditional expectations 
cannot be easily enforced in practice.
Rather, we penalize processes $M$ for which increments at every time step are non-zero on average. 
For an $(\Xcal,\Fcal^\Xcal)$-adapted process $M_{\varphi_2}^j$ and a mini-batch $\{x^i\}_{i=1}^m$ ($\sim\hat{\textbf{x}}$),
we define the martingale penalization for $\vM_{\varphi_2}$ as
\[
{p}_{\vM_{\varphi_2}}(\xhatm):=\frac{1}{mT}\sum_{j=1}^J\sum_{t=1}^{T-1}\Bigg|\sum_{i=1}^m \frac{\Delta_{t+1} M_{\varphi_2}^j(x^i)}{\sqrt{\text{Var}[M_{\varphi_2}^j]} + \eta}\Bigg|,
\]
where $\text{Var}[M]$ is the empirical variance of $M$ over time and batch, and $\eta>0$ is a small constant.
Third, we use the mixed normalization introduced in \eqref{Sink_W4}.
Each of the four terms is approximated by running the Sinkhorn algorithm on the cost $c_\varphi^\Kcal$ for an a priori fixed number of iterations $L$.

Altogether, we arrive at the following adversarial objective function
for COT-GAN:
\begin{equation}\label{eq:minibath_objective}
\What^{\textrm{mix},L}_{c_{\varphi}^\Kcal, \eps}(
    \xhatm, \xhatpm, \yhatm_\theta, \yhatpm_\theta) 
- {\lambda} p_{{\bf M}_{\varphi_2}}(\xhatm),
\end{equation}
where $\xhatm$ and $\xhatpm$ are empirical measures corresponding to two samples of the dataset, $\yhatm_\theta$ and $\yhatpm_\theta$ are the ones corresponding to two samples from $\nu_\theta$, and $\lambda$ is a positive constant. 
We update $\theta$ to decrease this objective,
and $\varphi$ to increase it.

While the generator $g_\theta:\mathcal{Z}\to\mathcal{X}$ acts as in classical GANs, the adversarial role here is played by $\vh_{\varphi_1}$ and $\vM_{\varphi_2}$. In this setting, the discriminator, parameterized by $\varphi$, learns a robust (worst-case) distance between the real data distribution $\mu$ and the generated distribution $\nu_\theta$, where the class of cost functions as in \eqref{eq:cot_cost} originates from causality. The algorithm is summarized in \cref{alg}.
Its time complexity scales as $\mathcal{O}((J+d)LTm^2)$ for each iteration. 

\begin{algorithm}[t]
\SetAlgoLined
\KwData{$\{{x}^i\}_{i=1}^{N}$ (real data), $\zeta$ (probability distribution on latent space $\Zcal$)}
\KwParam{$\theta_0$, $\varphi_{0}$, $m$ (batch size), $\eps$ (regularization parameter), $L$ (number of Sinkhorn iterations), $\alpha$ (learning rate), $\lambda$ (martingale penalty coefficient)}
\KwResult{$\theta$, $\varphi$}
 Initialize: $\theta \leftarrow \theta_0$, $\varphi \leftarrow \varphi_{0}$\\
 \For{$k = 1, 2,\dots$}{
    ~Sample $\{{x}^i\}_{i=1}^{m}$ and $\{{x}'^i\}_{i=1}^{m}$ from real data;\\[0.1cm]
    ~Sample $\{{z}^i\}_{i=1}^{m}$ and $\{{z}'^i\}_{i=1}^{m}$ from 
    ~$\zeta$;\\[0.1cm]
    ~$({y}_\theta^i,{y}_\theta'^i) \leftarrow (g_{\theta}({z}^i),g_{\theta}({z}'^i))$;\\[0.1cm]
    ~{Compute 
    $\What^{\textrm{mix},L}_{c_{\varphi}^\Kcal, \eps}
    (\xhatm, \xhatpm, \yhatm_\theta, \yhatpm_\theta)
    $ \eqref{Sink_W4} by the Sinkhorn 
    algorithm, with $c_{\varphi}^\Kcal$ given by \eqref{eq:cot_cost}}; \\
    ~${\varphi} \leftarrow {\varphi} + \alpha \nabla_\varphi\Big( \What^{\textrm{mix},L}_{c_{\varphi}^\Kcal, \eps}
  (\xhatm, \xhatpm, \yhatm_\theta, \yhatpm_\theta) - {\lambda} p_{{\bf M}_{\varphi_2}}(\xhatm)\Big)$; \\[0.1cm]
    ~Sample $\{{x}^i\}_{i=1}^{m}$ and $\{{x}'^i\}_{i=1}^{m}$ from real data;\\[0.1cm]
    ~Sample $\{{z}^i\}_{i=1}^{m}$ and $\{{z}'^i\}_{i=1}^{m}$ from  $\zeta$;\\[0.1cm]
    ~$({y}_\theta^i,{y}_\theta'^i) \leftarrow (g_{\theta}({z}^i),g_{\theta}({z}'^i))$;\\[0.1cm]
    ~{Compute 
    $\What^{\textrm{mix},L}_{c_{\varphi}^\Kcal, \eps}
    (\xhatm, \xhatpm, \yhatm_\theta, \yhatpm_\theta)
    $ \eqref{Sink_W4} by the Sinkhorn 
    algorithm, with $c_{\varphi}^\Kcal$ given by \eqref{eq:cot_cost}};
    ~$\theta \leftarrow \theta - \alpha \nabla_\theta\left( 
    \What^{\textrm{mix},L}_{c_{\varphi}^\Kcal, \eps}(\xhatm, \xhatpm, \yhatm_\theta, \yhatpm_\theta)\right)$;
}
 \caption{training COT-GAN by SGD}
 \label{alg}
\end{algorithm}

\section{Related work}\label{sec:related}
Early video generation literature focuses on dynamic texture modeling \cite{doretto2003dynamic,szummer1996temporal,wei2000fast}. Recent efforts in video generation within the GAN community have been devoted to designing GAN architectures of a generator and discriminator 
to tackle the spatio-temporal dependencies separately, e.g., \citep{vgan,tgan, mocogan}. VGAN \citep{vgan} explored a two-stream generator that combines a network for a static background and another one for moving foreground trained on the original GAN objective. TGAN \citep{tgan} proposed a new structure capable of generating dynamic background as well as a weight clipping trick to regularize the discriminator.
In addition to a unified generator, MoCoGAN \citep{mocogan} employed two discriminators to judge both the quality of frames locally and the evolution of motions globally.  

The broader literature of sequential data generation attempts to capture the dependencies in time by simply deploying recurrent neural networks
in the architecture \citep{crnngan,rcgan, biosignal, TimeGAN}. 
Among them, TimeGAN \citep{TimeGAN} demonstrated improvements in time series generation by adding a teacher-forcing component in the loss function.  
Alternatively, WaveGAN \cite{wavegan} adopted the causal structure of WaveNet \citep{wavenet}. 
Despite substantial progress made, existing sequential GANs are generally domain-specific. We therefore aim to offer a framework that considers (transport) causality in the objective function and is suitable for more general sequential settings.

Whilst our analysis is built upon \citep{cuturi2013sinkhorn} and \citep{GPC}, we remark two major differences between COT-GAN and  the algorithm in \citep{GPC}.
First, we consider a different family of costs.
While \citep{GPC} learns the cost function $c(f_\varphi(x),f_\varphi(y))$
by parameterizing $f$ with $\varphi$, the family of costs in COT-GAN is found by adding a causal component to $c(x,y)$ in terms of $\mathbf{h}_{\varphi_1}$ and $\mathbf{M}_{\varphi_2}$.
is the mixed Sinkhorn divergence we propose, which reduces biases in parameter estimation and can be used as a generic divergence for training IGMs not limited to time series settings. 

\section{Experiments}

\subsection{Time series}\label{sec:low_d_exp}
We now validate COT-GAN empirically\footnote{Code and data are available at \httpsurl{github.com/tianlinxu312/cot-gan}}. 
For times series that have a 
relatively small dimensionality $d$ but exhibit complex temporal structure, 
we compare COT-GAN with the following methods:
\textbf{TimeGAN} \citep{TimeGAN} as reviewed in \cref{sec:related};
\textbf{WaveGAN} \citep{wavegan} as reviewed in \cref{sec:related};
and \textbf{SinkhornGAN}, similar to \citep{GPC} with cost
$c(f_\varphi(x),f_\varphi(y))$ 
where $\varphi$ is trained to increase the mixed Sinkhorn 
divergence with weight clipping. 
All methods use $c(x,y)=\|x-y\|_2^2$.
The networks $h$ and $M$ in COT-GAN and $f$ in SinkhornGAN 
share the same architecture.
Details of models and datasets are in \cref{sec:low_d_exp_detail}.

\begin{figure}[t]
    \centering
    \includegraphics[width=\textwidth]{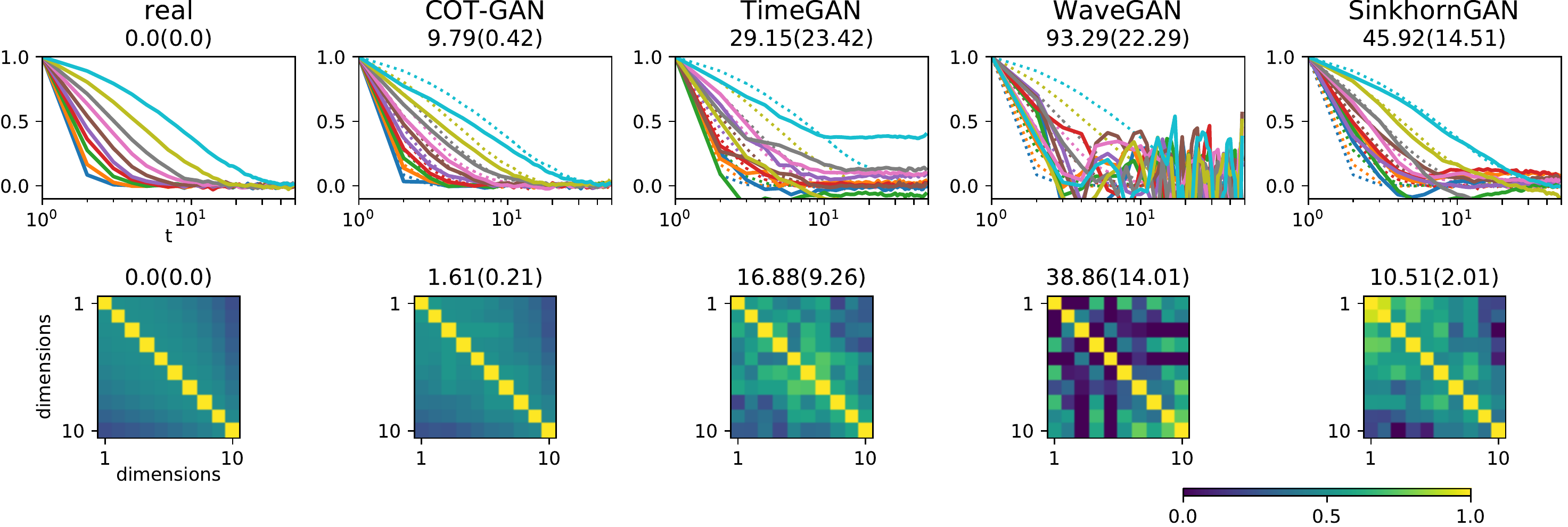}
    \caption{Results on learning the multivariate AR-1 process. 
    Top row shows the auto-correlation coefficient for each channel. 
    Bottom row shows the correlation coefficient between channels averaged over time.
    The numbers on top of each panel are the mean and standard deviation (in brackets) of the sum of the absolute difference between the correlation coefficients computed from real (leftmost) and generated samples for 16 runs with different random seeds. }
    \label{fig:exp_AR}
\end{figure}

\begin{figure}[t]
    \centering
    \includegraphics[width=\textwidth]{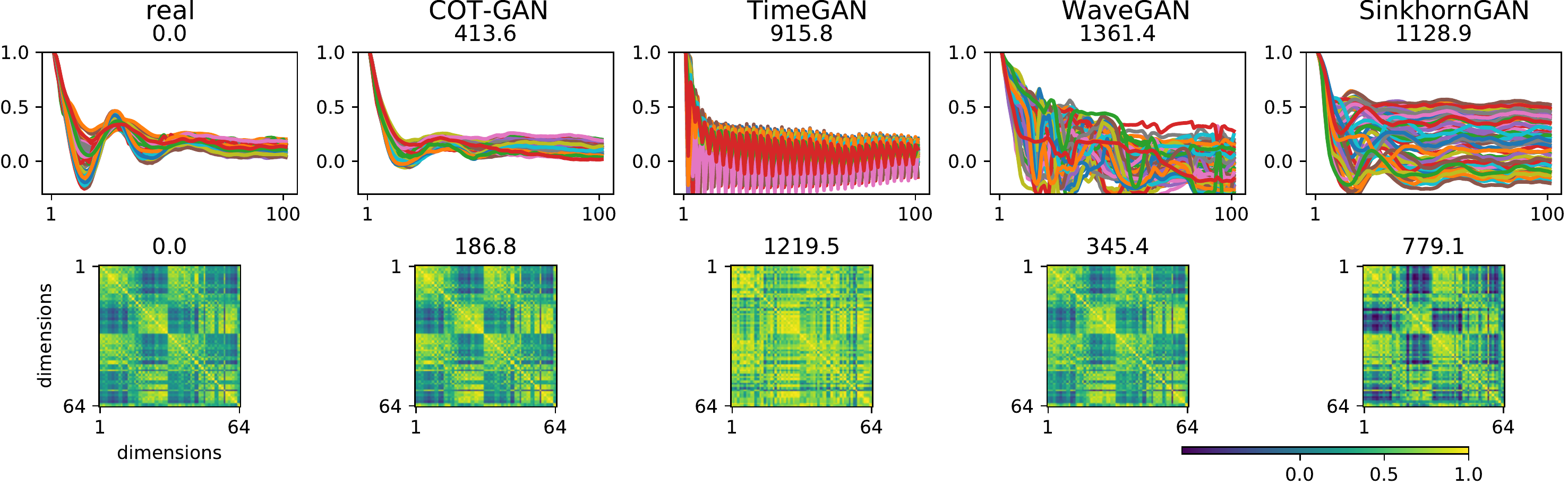}
    \caption{Results on EEG data. The same correlations as \cref{fig:exp_AR} are shown. 
    }
    \label{fig:exp_eeg}
\end{figure}

\paragraph{Autoregressive processes.}
We first test whether COT-GAN can learn temporal and spatial correlation
in a multivariate first-order auto-regressive process (AR-1).

For these experiments, we report two evaluation statistics: the sum of the absolute difference of the correlation coefficients between channels averaged over time, and the absolute difference between the correlation coefficients of real samples and those of generated samples. 
We evaluate the performance of each method by taking the mean and standard deviation of these two evaluation statistics over 16 runs with different random seeds.

In \cref{fig:exp_AR}, we show an example plot of results from a single run, as well as the evaluation statistics aggregated over all 16 runs on top of each panel. 
COT-GAN samples have correlation structures that best match the real data. 
Neither TimeGAN, WaveGAN nor SinkhornGAN captures the correlation structure for this dataset. The small standard deviation of the evaluation statistics demonstrates that COT-GAN is the most stable model at least in the AR-1 experiment since it produces similar results from each run of the model. 

\paragraph{Noisy oscillations.}
The noisy oscillation distribution 
is composed of sequences of 20-element arrays 
(1-D images) \citep{WenliangSahani2019neurally}. 
\Cref{fig:oscillation} in \cref{sec:low_d_exp_detail} shows data  
as well as generated samples by different training methods.
To evaluate performance, we estimate two attributes of the samples by Monte Carlo:
the marginal distribution of pixel values, 
and the joint distribution of the location 
at adjacent time steps. 
COT-GAN samples match the real data best.

\paragraph{Electroencephalography (EEG).} This dataset is from 
the UCI repository \citep{uci} and contains recordings from 43 healthy 
subjects each undergoing around 80 trials. Each data sequence 
has 64 channels and we model the first 100 time steps. 
We compare performance of COT-GAN with respect to other baseline models by investigating how well the generated samples match with the real data in terms of temporal and channel correlations, see Figure \ref{fig:exp_eeg}, and how the coefficient $\lambda$ affects sample quality, see \cref{sec:low_d_exp_detail}. 
COT-GAN generates the best samples compared with other baselines across two metrics. 

\begin{figure}
    \centering
    \includegraphics[width=\textwidth]{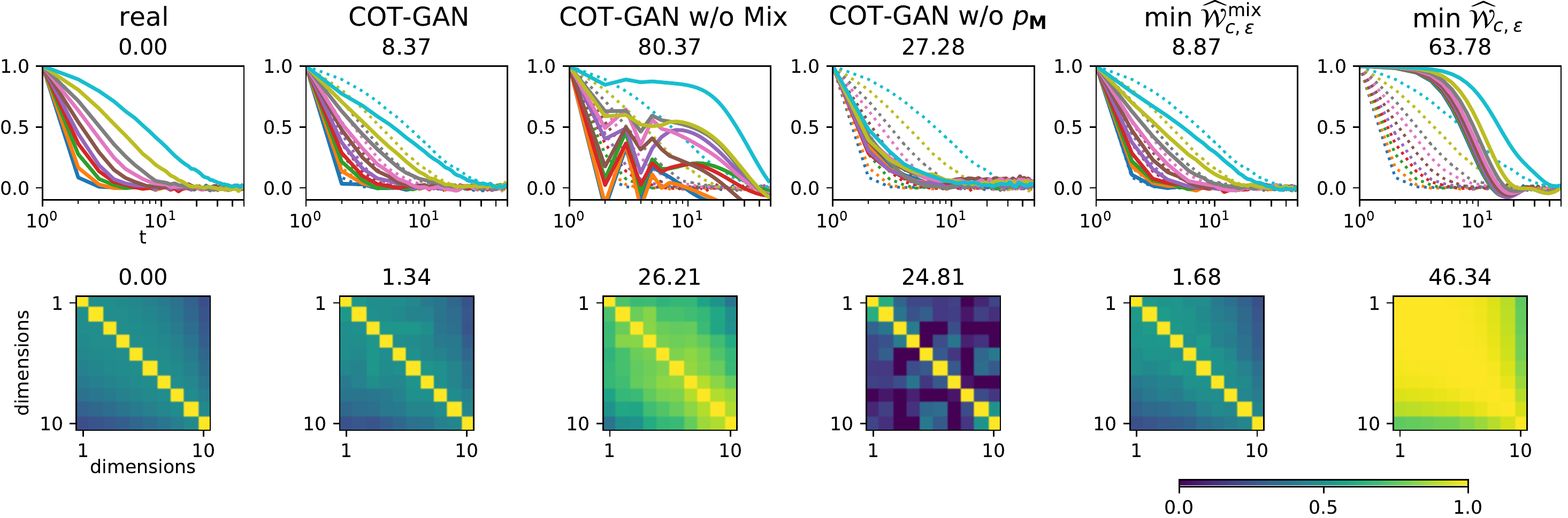}
    \caption{Ablation investigation.}
    \label{fig:ablation}
\end{figure}

In addition, we provide an ablation investigation of COT-GAN, in which we study the impact of the components of the model by excluding each of them in the multivariate AR-1 experiment. In \cref{fig:ablation}, we compare the real samples with COT-GAN, COT-GAN  using the original Sinkhorn divergence without the mixing, COT-GAN without the martingale penalty $p_{\mathbf{M}}$, direct minimization (without a discriminator) of the mixed and original Sinkhorn divergences from \eqref{Sink_W4} and \eqref{Sink_gpc}. We conclude that each component of COT-GAN plays a role in producing the best result in this experiment, and that the mixed Sinkhorn divergence is the most important factor for improvements in performance.

\subsection{Videos}
We train COT-GAN on animated Sprites \citep{disentangledvae,reed2015deep} and human action sequences \citep{ActionsAsSpaceTimeShapes_iccv05}.
We pre-process the Sprites sequences to have a sequence length of $T=13$, and the human action sequences to have length $T=16$. Each frame has dimension $64\times64\times3$. We employ the same architecture for the generator and discriminator to train both datasets. 
Both the generator and discriminator consist of a generic LSTM with 2-D convolutional layers.  
Details of the data pre-processing, GAN architectures, hyper-parameter settings, and training techniques are reported in Appendix \ref{apx:gan_structure}. 
 
Baseline models chosen for the video datasets are \textbf{MoCoGAN} from \citep{mocogan}, and direct minimization of the mixed Sinkhorn divergence \eqref{Sink_W4}, as it achieves a good result when compared to the other methods addressed in Figures \ref{fig:exp_AR} and \ref{fig:ablation}.  We show the real data and generated samples from COT-GAN side by side in \cref{human_action_results}. Generated samples from all methods, without cherry-picking, are provided in Appendix \ref{apx:humanactionresults}. The evaluation metrics we use to assess model performance are the Fr{\'e}chet Inception Distance (FID) \citep{fid} which compares individual frames, the Fr{\'e}chet Video Distance (FVD) \citep{fvd} which compares the video sequences as a whole by mapping samples into features via pretrained 3D convolutional networks, and 
their kernel counterparts (KID, KVD) \citep{kid}. Previous studies suggest that FVD correlates better with human judgement than KVD for videos \citep{fvd}, 
whereas KID correlates better than FID on images \citep{hype}. 

\begin{figure}[ht]
    \centering
    \includegraphics[width=0.45\textwidth]{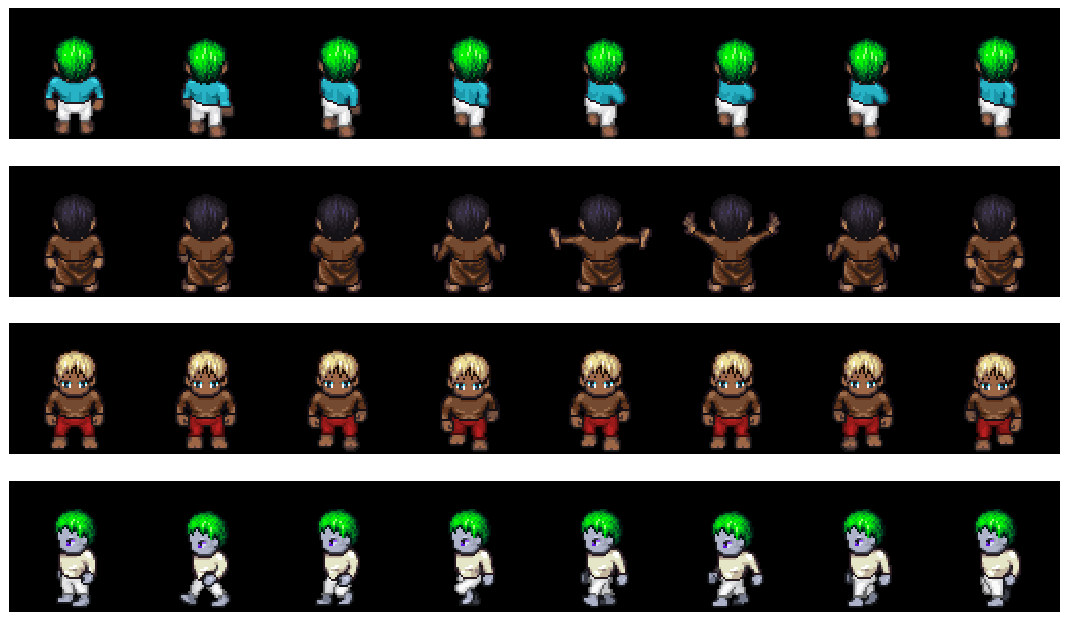}
    \includegraphics[width=0.45\textwidth]{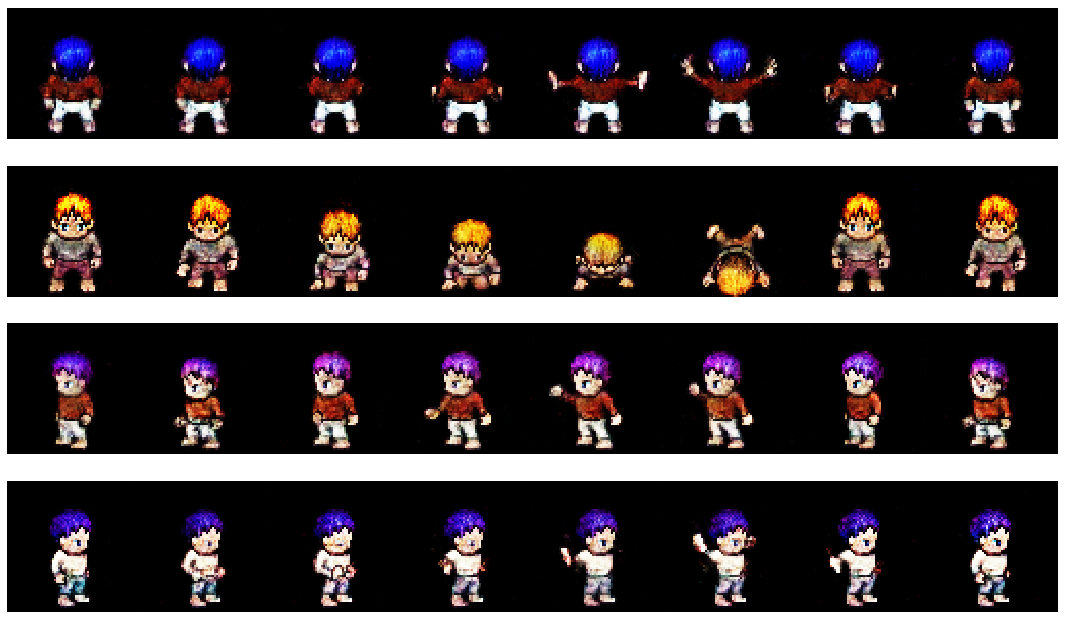}\\
    \vspace{1em}
    \includegraphics[width=0.45\textwidth]{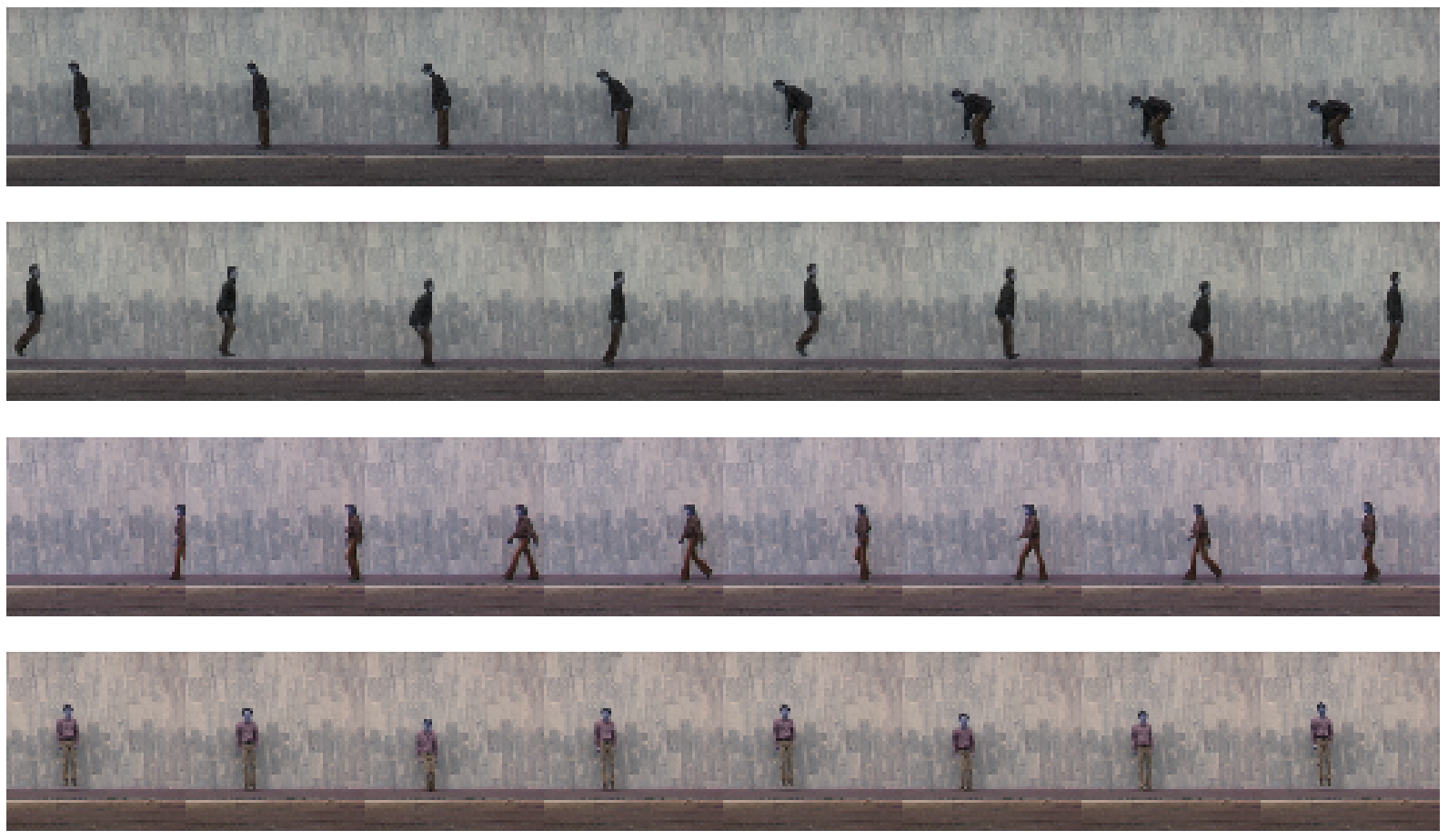}
    \includegraphics[width=0.45\textwidth]{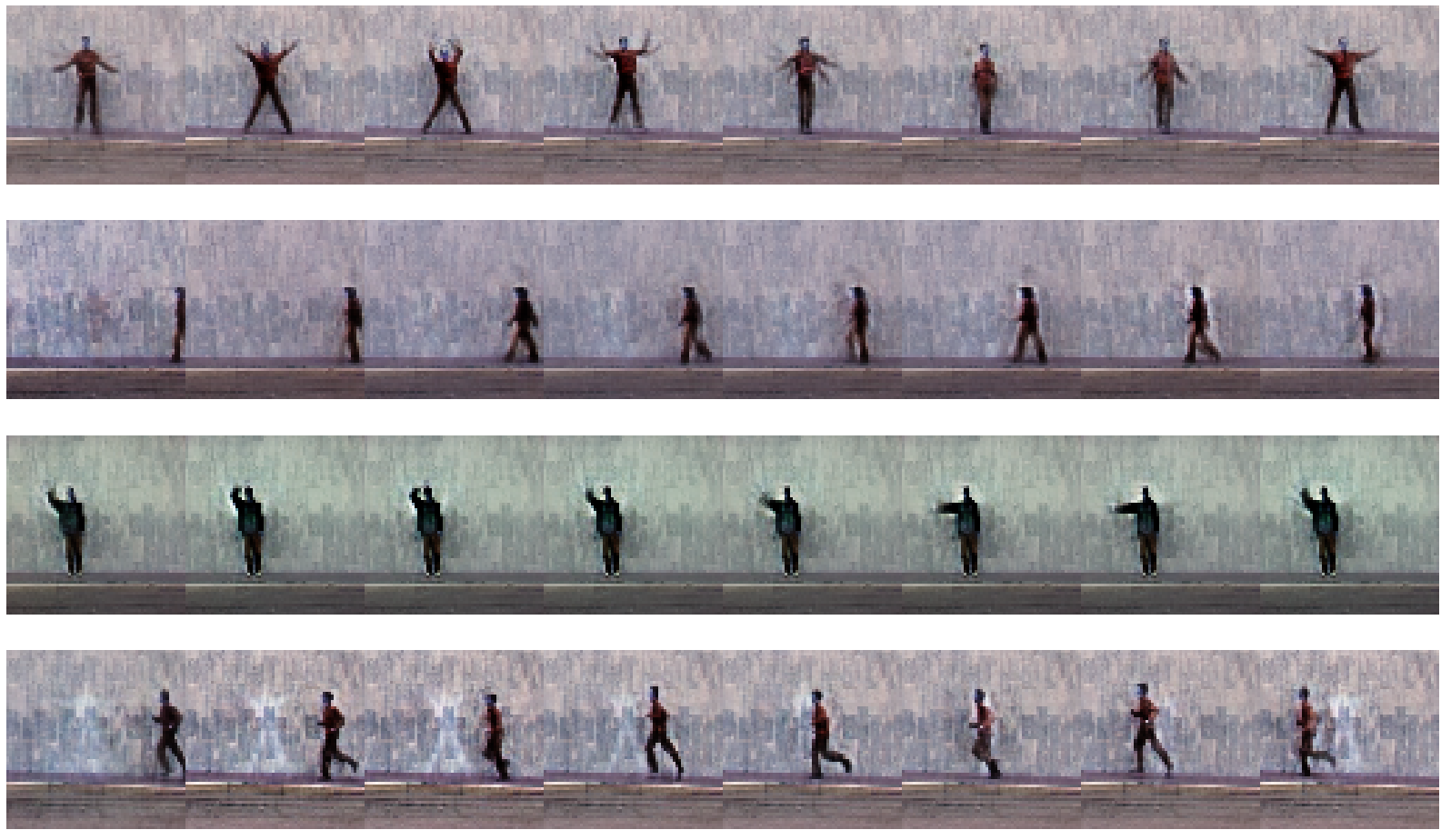}
    \caption{Animated (top) and human (bottom) action videos. 
    Left column reports real data samples, and right column 
 samples from COT-GAN.}
    \label{human_action_results}
\end{figure} 

\begin{table}[ht]
\centering
\caption{Evaluations for video datasets. Lower value indicates better sample quality.}
\begin{tabular}[t]{lcccc}
\hline
\hline
\bf{Sprites}   &   FVD     &      FID      &       KVD  &      KID \\
\hline
MoCoGAN  &   1\,108.2    &   280.25       &   146.8    &  0.34       \\
\scalebox{1}{$\min \Wcemix$}&   498.8       &   \bf{81.56}   &      83.2  &  \bf{0.078} \\
COT-GAN                     & \bf{458.0}    &   84.6      & \bf{66.1}  &  0.081    \\

\hline

\bf{Human actions}          &               &               &            &            \\
\hline
MoCoGAN   &   1\,034.3       &   151.3       &   89.0     &  0.26   \\
\scalebox{1}{$\min \Wcemix$}&   507.6  &   120.7       &  \bf{34.3} &  0.23 \\
COT-GAN       &   \bf{462.8}       &   \bf{58.9}   &  43.7     &  \bf{0.13}\\
\hline
\hline
\vspace{1pt}
\end{tabular}
\label{tab:compare1}
\end{table}

In Table \ref{tab:compare1} the evaluation scores are estimated using 10,000 generated samples. 
For Sprites, COT-GAN performs better than the other two methods on FVD and KVD. However, minimization of the mixed Sinkhorn divergence produces slightly better FID and KID scores when compared to COT-GAN. The results in \citep{fvd} suggest that
FID better captures the frame-level quality, while FVD is better suited for the temporal coherence in videos.
For the human action dataset, COT-GAN is the best performing method across all metrics except for KVD. 

\section{Discussion}
With the present paper, we introduce the use of causal transport theory in the machine learning literature. As already proved in other research fields, we believe it may have a wide range of applications here as well. The performance of COT-GAN already suggests that 
constraining the transport plans to be causal is a promising direction  
for generating sequential data. 
The approximations we introduce, such as the mixed Sinkhorn distance \eqref{Sink_W4} and truncated sum in
\eqref{eq:L_set}, are sufficient to produce good experimental results, 
and provide opportunities for more theoretical analyses in future studies.
Directions of future development 
include ways to learn from data with flexible lengths, extensions to conditional COT-GAN, and improved methods to enforce the martingale property for $\vM$ and better parameterize the causality constraint. 

\section{Broader impact}
The COT-GAN algorithm introduced in this paper is suitable to generate sequential data, when the real dataset consists of i.i.d. sequences or of stationary time series.
It opens up doors to many applications that can benefit from time series synthesis.
For example, researchers often do not have access to abundant training data due to privacy concerns, high cost, and data scarcity. This hinders the capability of building accurate predictive models. 

Ongoing research
is aimed at developing a modified COT-GAN algorithm to generate financial time series.  The high non-stationarity of financial data requires different features and architectures, whilst causality when measuring distances between sequences remains the crucial tool.
The application to market generation is of main interest for the financial and insurance industry, for example in model-independent pricing and hedging, portfolio selection, risk management, and stress testing. In broader scientific research, our approach can be used to 
estimate from data the parameters of simulation-based models that describe 
physical processes. These models can be, for instance, differential 
equations describing neural activities, compartmental models 
in epidemiology, and chemical reactions involving multiple 
reagents.

\section*{Acknowledgments and Disclosure of Funding}
BA thanks the financial support from the Erwin Schr\"odinger Institute during the thematic programme on Optimal Transport (May 2019, Vienna). This material is based upon work supported by Google Cloud.  LKW is supported by the Gatsby Charitable Foundation.

\printbibliography

\clearpage
\appendix

\begin{center}
\Large {\textbf{\mytitle:\\Supplementary material}}
\end{center}

\section{Specifics on regularized Causal Optimal Transport}

\subsection{Limits of regularized Causal Optimal Transport}\label{sect.app.lim}
In this section we prove the limits stated in Remark~\ref{rem.lim}.

\begin{lem}
Let $\mu$ and $\nu$ be discrete measures, say on path spaces $\mathbb{X}^T$ and $\mathbb{Y}^T$, with $|\mathbb{X}|=m$ and $|\mathbb{Y}|=n$. Then
\[
\mathcal{K}_{c, \eps}(\mu, \nu)   \xrightarrow[\eps\to 0]{}  \mathcal{K}_{c}(\mu, \nu).
\]
\end{lem}

\begin{proof} We mimic the proof of Theorem~4.5 in \cite{ABJ}, and note that the entropy of any $\pi\in\Pi(\mu,\nu)$  is uniformly bounded:
\begin{equation}\label{eq_bdd_ent}
0\leq H(\pi) \leq C:=m^Tn^T e^{-1}. 
\end{equation}
This yields
\begin{eqnarray}\label{eq_three_inequalities}
\begin{split}\inf_{\pi\in\Pi^\Kcal(\mu,\nu)} \E^\pi[c] - \eps\, C +\eps H(\pi^\Kcal_{c,\eps}(\mu,\nu)) &\leq& \inf_{\pi\in\Pi^\Kcal(\mu,\nu)} \left\{\E^\pi[c] - \eps\, H(\pi) \right\}+\eps H(\pi^\Kcal_{c,\eps}(\mu,\nu))\\
&\leq&
\inf_{\pi\in\Pi^\Kcal(\mu,\nu)} \E^\pi[c]+\eps H(\pi^\Kcal_{c,\eps}(\mu,\nu)).
\end{split}
\end{eqnarray}
Now, note that $\inf_{\pi\in\Pi^\Kcal(\mu,\nu)} \left\{\E^\pi[c] - \eps\, H(\pi) \right\}= \mathcal{K}_{c, \eps}(\mu, \nu)-\eps H(\pi^\Kcal_{c,\eps}(\mu,\nu))$, and that, for $\eps\to0$, the LHS and RHS in \eqref{eq_three_inequalities} both tend to $\mathcal{K}_{c}(\mu, \nu)$. 
\end{proof}

\begin{lem}
Let $\mu$ and $\nu$ be discrete measures. Then
\[
\mathcal{K}_{c, \eps}(\mu, \nu) \xrightarrow[\eps\to \infty]{} \E^{\mu\otimes\nu}[c(x,y)].
\]
\end{lem}
\begin{proof} Being $\mu$ and $\nu$ discrete, $\E^\pi[c]$ is uniformly bounded for $\pi\in\Pi^\Kcal(\mu,\nu)$.
Therefore, for $\eps$ big enough, the optimizer in $\mathcal{P}^\Kcal_{c, \eps}(\mu, \nu)$ is $\hat\pi:=\argmax_{\pi\in\Pi^\Kcal(\mu,\nu)}H(\pi)=\mu\otimes\nu$, the independent coupling, for which $H(\mu\otimes\nu)=H(\mu)+H(\nu)$; see \cite{CT91} and \cite{G63}. Therefore, for $\eps$ big enough, we have $\mathcal{K}_{c, \eps}(\mu, \nu) = \E^{\mu\otimes\nu}[c(x,y)]$.
\end{proof}

\subsection{Reformulation of the COT problem}\label{sec:cot_prob_proof}
\begin{proof}
  The causal constraint \eqref{causalhM} can be expressed using the following 
  characteristic function:   
  \begin{equation}\label{eq.supl}
  \sup_{l\in\Lcal(\mu)}\E^{\pi}[l(x,y)]\; =
      \begin{cases}
          0 & \text{if $\pi$ is causal;}\\
          +\infty  & \text{otherwise.}
      \end{cases} 
  \end{equation}
  This allows to rewrite \eqref{eq:reg_cot_problem} as
  \begin{align*}
  \mathcal{P}^{\Kcal}_{c, \eps}(\mu, \nu) 
  &= \inf_{\pi \in \Pi(\mu, \nu)}\left\{\E^{\pi}\left[c(x,y)\right]- \eps H(\pi)+ \sup_{l\in\Lcal(\mu)} \E^{\pi}[l(x,y)] \right\} \\
  &= \inf_{\pi \in \Pi(\mu, \nu)}\sup_{l\in\Lcal(\mu)} \left\{\E^{\pi}\left[c(x,y)+l(x,y) \right]- \eps H(\pi)\right\}\\[0.2cm]
  &= \sup_{l\in\Lcal(\mu)}\inf_{\pi \in \Pi(\mu, \nu)} \left\{\E^{\pi}\left[c(x,y)+l(x,y) \right]- \eps H(\pi)\right\}\\[0.2cm]
  &= \sup_{l\in\Lcal(\mu)} \mathcal{P}_{c+l, \eps}(\mu, \nu),
  \end{align*}
  where the third equality holds by the min-max theorem, thanks to convexity of $\Lcal(\mu)$, and convexity and compactness of $\Pi(\mu,\nu)$. 
\end{proof}

\subsection{Sinkhorn divergence at the level of mini-batches}\label{sect.app.sink}   

\paragraph{Empirical observation of the bias in Example~\ref{bias_example}.} 
In the experiment mentioned in \cref{bias_example}, we consider a set of distributions $\nu$'s as sinusoids with random phase, frequency and amplitude. 
We let $\mu$ be one element in this set whose amplitude is uniformly distributed
between minimum 0.3 and maximum 0.8. On the other hand, for each $\nu$, 
the amplitude is uniformly distributed 
between the same minimum 0.3 and a maximum that lies in $\{0.4,0.5,\dots,1.2\}$. Thus, the only parameter of the distribution being varied is the 
maximum amplitude. We may equivalently take the maximum amplitude as a single $\theta$ 
that parameterizes $\nu_\theta$, so that $\mu=\nu_{0.8}$.  \cref{fig:sink_bias} illustrates that 
the sample Sinkhorn divergence \eqref{Sink_gpc}  (or regularized distance \eqref{eq:Wce}) does  
not recover the optimizer $0.8$, while the proposed mixed Sinkhorn divergence \eqref{Sink_W4} does.

\paragraph{Comparison of various implementations.}
Motivated by \citet{biaswasserstein}, \citet{mixed_sink} address the problem of bias in the mini-batch gradients of Wasserstein distance by proposing a mini-batch Sinkhorn divergence that is closely related to \eqref{Sink_W4}.  
We denote the implementation of a mini-batch Sinkhorn divergence in \citet{mixed_sink} as 
\begin{align*}
\widehat{\mathcal{W}}_{c,\epsilon}^{6}
 &:= \mathcal{W}_{c, \eps}(\hat{\textbf{x}}, \hat{\textbf{y}}_\theta)
+\mathcal{W}_{c, \eps}(\hat{\textbf{x}}, \hat{\textbf{y}}_\theta')
+\mathcal{W}_{c, \eps}(\hat{\textbf{x}'}, \hat{\textbf{y}}_\theta)
+\mathcal{W}_{c, \eps}(\hat{\textbf{x}}', \hat{\textbf{y}}_\theta') \\
& \quad -2\mathcal{W}_{c, \eps}(\hat{\textbf{x}}, \hat{\textbf{x}}')
-2\mathcal{W}_{c, \eps}(\hat{\textbf{y}}_\theta, \hat{\textbf{y}}_\theta').
\end{align*}

In addition to \eqref{Sink_gpc} and \eqref{Sink_W4}, we further consider other possible variations of the Sinkhorn divergence at the level of mini-batches, including 
\begin{equation*}
   \widehat{\mathcal{W}}_{c,\epsilon}^{3} := 2\mathcal{W}_{c, \eps}(\hat{\textbf{x}}, \hat{\textbf{y}}_\theta) -\mathcal{W}_{c, \eps}(\hat{\textbf{x}}, \hat{\textbf{x}}') -\mathcal{W}_{c, \eps}(\hat{\textbf{y}}_\theta, \hat{\textbf{y}}_\theta')
\end{equation*}
and
\begin{align*}
\widehat{\mathcal{W}}_{c,\epsilon}^{8} &:= \mathcal{W}_{c, \eps}(\hat{\textbf{x}}, \hat{\textbf{y}}_\theta)
+\mathcal{W}_{c, \eps}(\hat{\textbf{x}}, \hat{\textbf{y}}_\theta')
+\mathcal{W}_{c, \eps}(\hat{\textbf{x}'}, \hat{\textbf{y}}_\theta)
+\mathcal{W}_{c, \eps}(\hat{\textbf{x}}', \hat{\textbf{y}}_\theta') \\
& \quad -\mathcal{W}_{c, \eps}(\hat{\textbf{x}}, \hat{\textbf{x}}')
-\mathcal{W}_{c, \eps}(\hat{\textbf{y}}_\theta, \hat{\textbf{y}}_\theta')
-\mathcal{W}_{c, \eps}(\hat{\textbf{x}}, \hat{\textbf{x}})
-\mathcal{W}_{c, \eps}(\hat{\textbf{y}}_\theta, \hat{\textbf{y}}_\theta).
\end{align*}

The superscripts in $\widehat{\mathcal{W}}_{c,\epsilon}^3$, $\widehat{\mathcal{W}}_{c,\epsilon}^6$ and $\widehat{\mathcal{W}}_{c,\epsilon}^8$ indicate the number of terms used in the mini-batch implementation of the Sinkhorn divergence. In the same spirit, our choice of mixed Sinkhorn divergence $\widehat{\mathcal{W}}_{c,\epsilon}^{\text{mix}}$ corresponds to $\widehat{\mathcal{W}}_{c,\epsilon}^4$. 

We compare the performance of all the variations in the low-dimensional applications of multivariate AR-1 and 1-D noisy oscillation (see Appendix \ref{app:experiment} for experiment details) in \cref{fig:ar_w6} and \cref{fig:sine_w6}, and in the high-dimensional applications of Sprite animations and the Weizmann Action dataset in \cref{tab:compare}.  The superscripts on COT-GAN correspond to the Sinkhorn divergence used in the experiments. We replace the COT-GAN objective \eqref{Sink_W4} with \eqref{Sink_gpc} in the experiment of $\text{COT-GAN}^2$, with $\widehat{\mathcal{W}}_{c,\epsilon}^3$ in $\text{COT-GAN}^3$, with $\widehat{\mathcal{W}}_{c,\epsilon}^6$ in $\text{COT-GAN}^6$, and with $\widehat{\mathcal{W}}_{c,\epsilon}^8$ in $\text{COT-GAN}^8$, respectively.  

\begin{figure}[ht!]
    \centering
    \includegraphics[width=\textwidth]{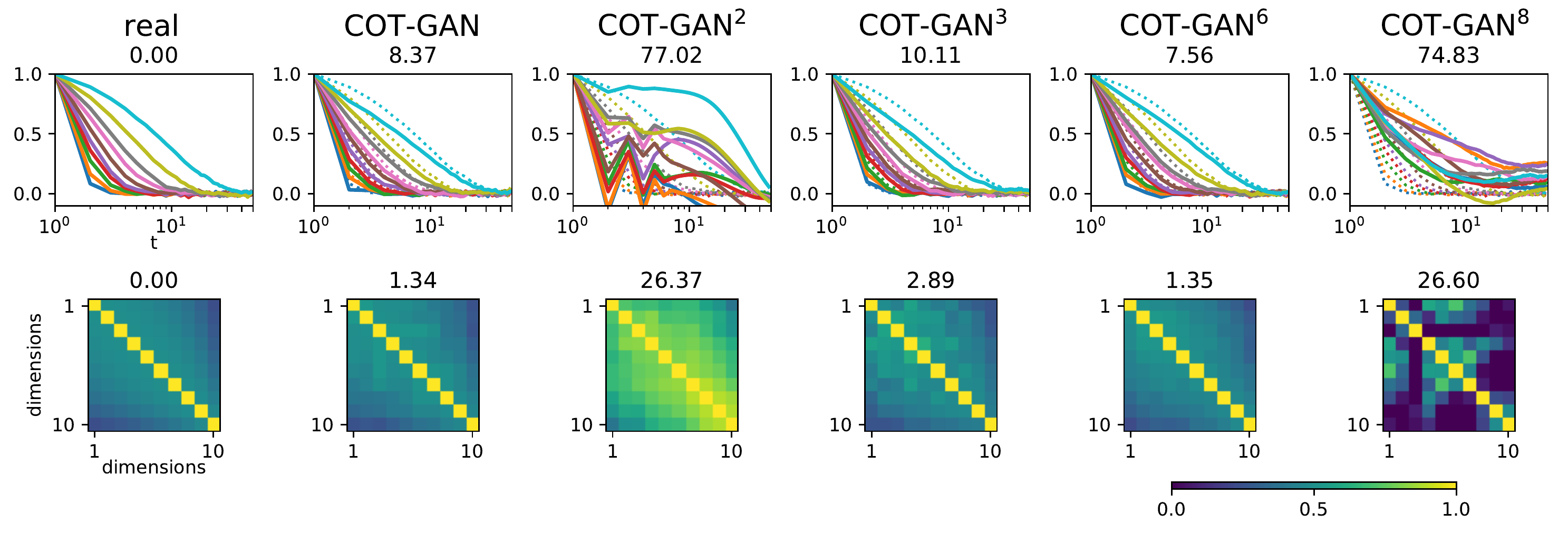}
    \caption{Results on learning the multivariate AR-1 process.}
    \label{fig:ar_w6}
\end{figure}

\begin{figure}
    \centering
    \includegraphics[width=\textwidth]{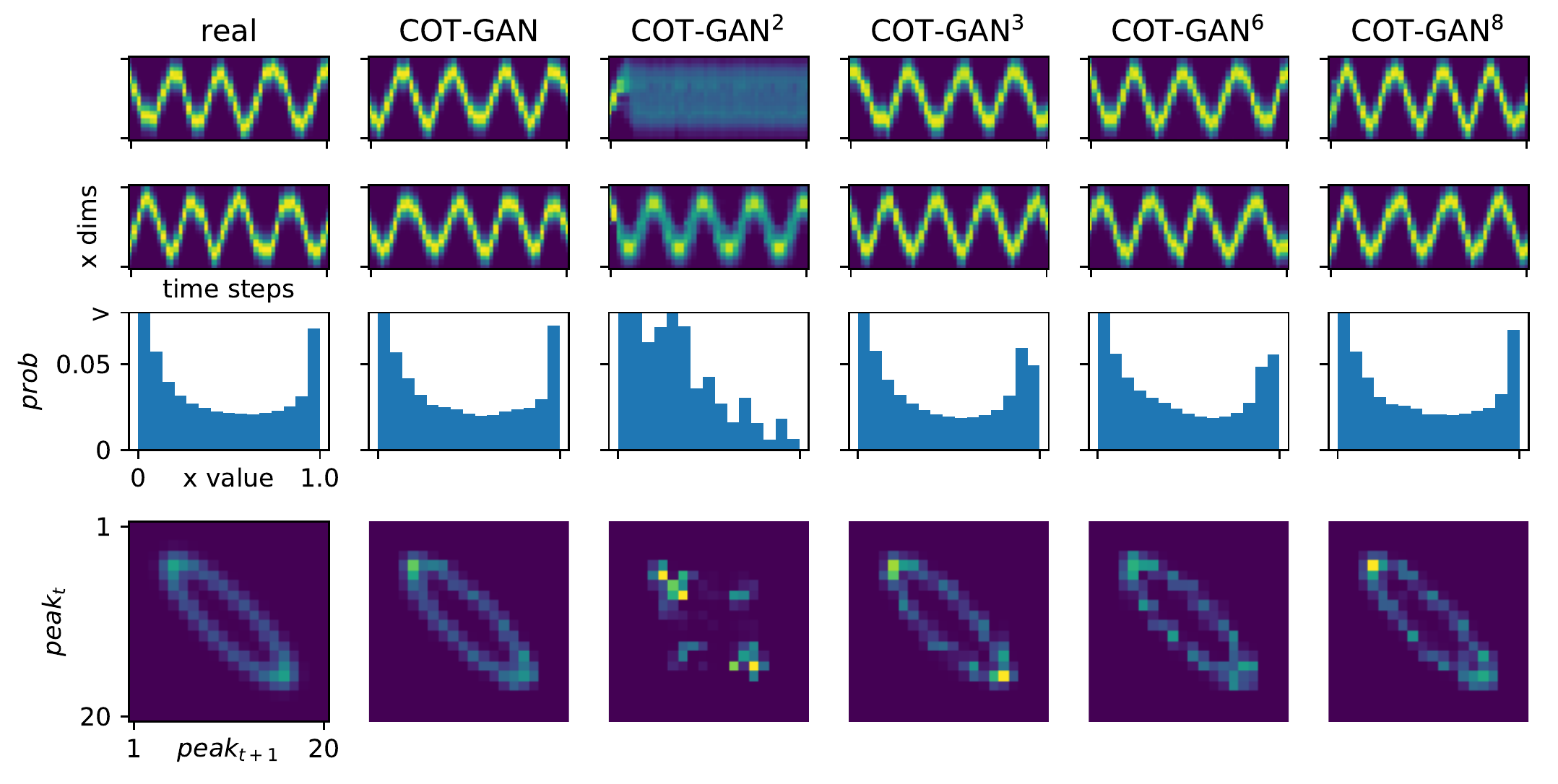}
    \caption{1-D noisy oscillation. Top two rows show two samples from the data distribution and generators trained by different methods. Third row shows marginal distribution of pixels values (y-axis clipped at 0.07 for clarity). Bottom row shows joint distribution of the position of the oscillation at adjacent time steps.}
    \label{fig:sine_w6}
\end{figure}

In the low-dimensional experiments, COT-GAN outperforms $\text{COT-GAN}^6$ on the 1-D noisy oscillation, but underperforms it on the multivariate AR-1 experiment.  Both COT-GAN and $\text{COT-GAN}^6$ obtain significantly better results than all other variations of the mini-batch Sinkhorn divergence. Given the low-dimensional results, we only compare COT-GAN and $\text{COT-GAN}^6$ in the high-dimensional experiments. As shown in Table \ref{tab:compare}, COT-GAN performs the best in all evaluation metrics except for KVD for Sprites animation.  Both COT-GAN and $\text{COT-GAN}^6$ perform better than MoCoGAN in these two tasks. However, because $\text{COT-GAN}^6$requires more mini-batches in the computation, it is about 1.5 times slower than COT-GAN.

\begin{table}[ht]
\centering
\caption{Evaluations for video datasets. Lower value indicates better sample quality.}
\begin{tabular}[t]{lcccc}
\hline
\hline
\bf{Sprites}   &   FVD     &      FID      &       KVD  &      KID \\
\hline
MoCoGAN  &   1\,108.2    &   280.25       &   146.8    &  0.34       \\
$\text{COT-GAN}^6$ &   620.1       &   109.1  &      \bf{64.5}  &  0.091 \\
COT-GAN                     & \bf{458.0}    &   \bf{84.6}      &  66.1  &  \bf{0.081}    \\

\hline

\bf{Human actions}          &               &               &            &            \\
\hline
MoCoGAN   &   1\,034.3       &   151.3       &   89.0     &  0.26   \\
$\text{COT-GAN}^6$ &   630.8  &   109.2       &  46.79 &  0.19 \\
COT-GAN       &   \bf{462.8}       &   \bf{58.9}   &  \bf{43.7}     &  \bf{0.13}\\
\hline
\hline
\vspace{1pt}
\end{tabular}
\label{tab:compare}
\end{table}

\paragraph{The MMD limiting case.}
In the limit $\eps\to\infty$, \citet{GPC} showed that
$\Wcal_{c,\eps}(\mu,\nu)\to\MMD_{-c}(\mu,\nu)$ under the kernel
defined by $-c(x,y)$. Here we want to point out an interesting fact about the limiting behavior of the mixed Sinkhorn divergence.

\begin{rem}
Given distributions of mini-batches $\xhatm$ and $\yhatm$ formed by samples
from $\mu$ and $\nu$, respectively, 
in the limit $\eps\to\infty$,
the Sinkhorn divergence $\What_{c,\eps}(\xhatm, \yhatm)$ converges to a biased estimator of 
$\MMD_{-c}(\mu,\nu)$;
given additional $\xhatpm$ and $\yhatpm$ from $\mu$ and $\nu$, respectively, 
the mixed Sinkhorn divergence 
$\What^{\text{mix}}_{c,\eps}(\xhatm,\xhatpm,\yhatm,\yhatpm)$ converges to 
an unbiased estimator of $\MMD_{-c}(\mu,\nu)$.
\end{rem}
\begin{proof}
The first part of the statement relies on the fact that $\MMD_{-c}(\xhatm,\yhatm)$ is a biased estimator of $\MMD_{-c}(\mu,\nu)$. Indeed, we have
\[
\What_{c,\eps}(\xhatm, \yhatm)\stackrel{\eps\to\infty}{\longrightarrow}
\MMD_{-c}(\xhatm,\yhatm) 
 = - \frac{1}{m^2}\sum_{i=1}^m \sum_{j=1}^m [c(x^i,x^j) + c(y^i,y^j) - 2c(x^i,y^j)].
\]
Now note that
\begin{align*}
\frac{1}{m^2}\sum_{i=1}^m \sum_{j=1}^m \E [c(x^i,x^j)] &=
\frac{1}{m^2}
\left[\sum_{i=1}^m\E_\mu[c(x^i,x^i)] + \sum_{i\ne j} \E_{\mu\otimes\mu}[c(x^i, x^j)]\right]\\
& = \frac{m-1}{m} \E_{\mu\otimes\mu}[c(x, x') ],
\end{align*}
where we have used the fact that $c(x^i,x^i)=0$. 
A similar result holds for the sum over $c(y^i,y^j)$. 
On the other hand, $\frac{1}{m^2}\sum_{ij}\E[c(x^i,y^j)]=\E_{\mu\otimes\nu}[c(x,y)]$.
Therefore
\begin{align*}
\E\MMD_{-c}(\xhatm,\yhatm) &= -\frac{m-1}{m}[
    \E_{\mu\otimes\mu}[c(x, x')] + \E_{\nu\otimes\nu}[c(y, y')]] + 2\E_{\mu\otimes\nu}[c(x,y)]\\
    &\ne \MMD_{-c}(\mu,\nu),
\end{align*}
which completes the proof of the first part of the statement. 

For the second part, note that 
$\Wcal_{c,\eps}(\mu,\nu)\to\E_{\mu\otimes\mu}[c(x,x')]$ as $\eps\to\infty$
\citep[Theorem 1]{GPC}, thus
\begin{align*}
\What^{\text{mix}}_{c,\eps}(\xhatm, \xhatpm, \yhatm, \yhatpm) &\to 
\E_{\xhatm\otimes\yhatm}[c(x,y)] + \E_{\xhatpm\otimes\yhatpm}[c(x',y')] 
- \E_{\xhatm\otimes\xhatpm}[c(x,x')] - \E_{\yhatm\otimes\yhatpm}[c(y,y')] \\
&=\frac{1}{m^2}\sum_{i=1}^m\sum_{j=1}^m[c(x^i,y^i) + c(x'^i,y'^i) - c(x^i,x'^i) - c(y^i,y'^i)].
\end{align*}
The RHS is an unbiased estimator of $\MMD$, since its expectation is
\begin{align*}
\E_{\mu\otimes\nu}[c(x,y)] + \E_{\mu\otimes\nu}[c(x',y')] 
- \E_{\mu\otimes\mu}[c(x,x')] - \E_{\nu\otimes\nu}[c(y,y')] &=\MMD_{-c}(\mu,\nu).
\end{align*}
\end{proof}

The mixed divergence may still be a biased estimate of the true 
Sinkhorn divergence. 
However, in the experiment 
of Example~\ref{bias_example} we note that the minimum is reached for the parameter $\theta$ close to the real one (\cref{fig:sink_bias}, bottom). 

\section{Experimental details} \label{app:experiment}
\subsection{Low dimensional time series}\label{sec:low_d_exp_detail}
Here we describe details of the experiments in \cref{sec:low_d_exp}.
\paragraph{Autoregressive process.}
The generative process to obtain data $\vxt$ for the autoregressive process is
$$
\mathbf{x}_t=\mathbf{Ax}_{t-1}+\vzeta_t, \quad \vzeta_t\iidsim\mathcal{N}(0,\boldsymbol{\Sigma}),
\quad \boldsymbol{\Sigma} = 0.5\mI + 0.5,
$$
where $\mA$ is diagonal with ten values evenly spaced between $0.1$ and $0.9$.
We initialize $\vx_0$ from a
10-dimensional standard normal, 
and ignore the data in the 
first 10 time steps so that
the data sequence 
begins with a more or 
less stationary distribution.
We use $\lambda=0.1$ and $\eps=10.0$ for this experiment. 

\begin{figure}[t]
    \centering
    \includegraphics[width=\textwidth]{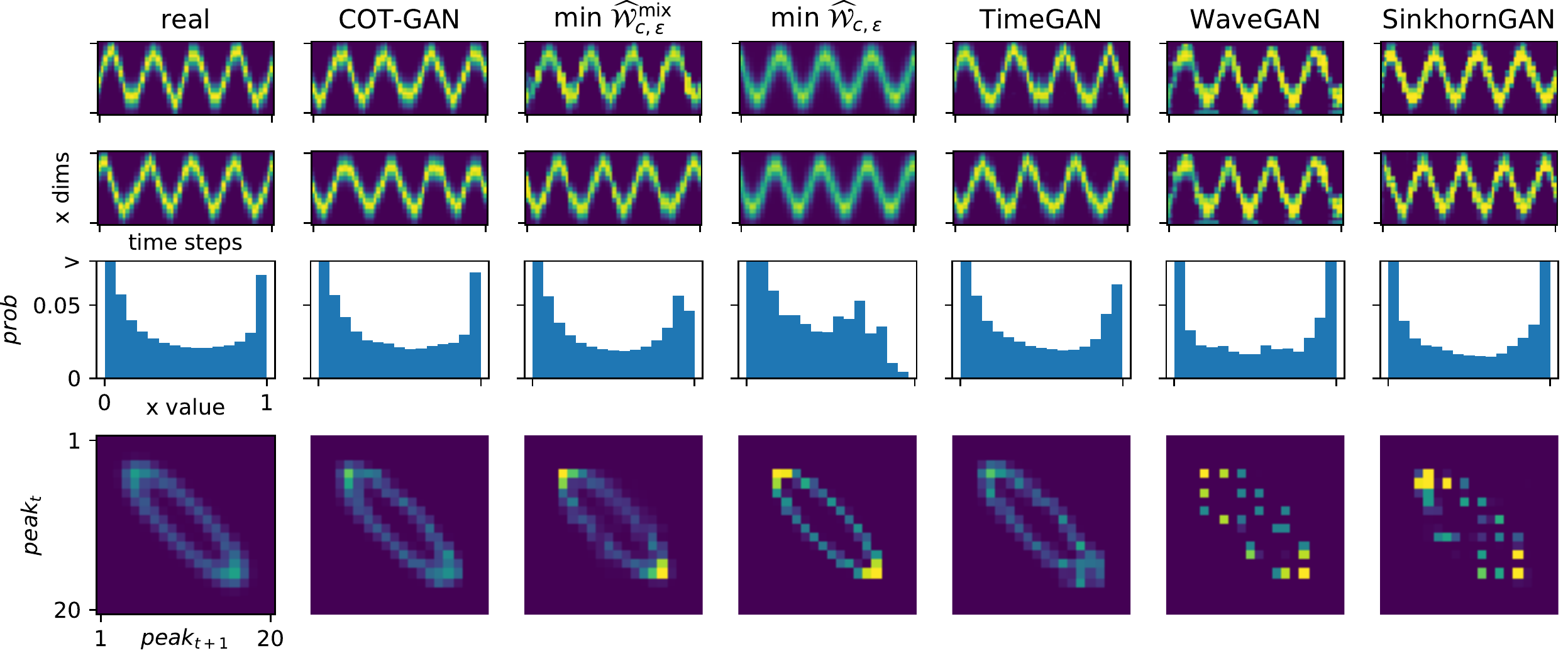}
    \caption{1-D noisy oscillation. 
    Same distributions as in \ref{fig:sine_w6} are shown.
    }
    \label{fig:oscillation}
\end{figure}

\paragraph{Noisy oscillation.}

This dataset comprises paths simulated from a noisy, nonlinear dynamical system.
Each path is represented as a sequence of $d$-dimensional arrays, $T$ time steps long, and can be displayed as a $d\times T$-pixel image for visualization.
At each discrete time step $t\in\{1,\dots,T\}$, data at time $t$, given by $\vx_t\in[0,1]^{d}$, is determined by the position of a ``particle'' following
noisy, nonlinear dynamics. 
When shown as an image, 
each sample path appears visually as a ``bump'' 
travelling rightward, moving up and down in a zig-zag pattern as shown in \cref{fig:oscillation} (top left).

More precisely, the state of the particle at time $t$ 
is described by its position and velocity
$\vs_t=(s_{t,1},s_{t,2}) \in \sR^2$, and evolves 
according to 
\begin{gather*}
    \vs_t = \vf(\vs_{t-1}) + \vzeta_t, \quad \vzeta_t =\mathcal{N}(0, 0.1\mI),
    \quad \\
    \vf(\vstm) = c_t\mathbf{A}\vstm; \quad c_t = \frac{1}{\|\vstm\|_2 \exp(-4(\|\vstm\|_2 -0.3)+1)},\\
\end{gather*}
where $\mA\in \sR^{2\times 2}$ is a rotation matrix, and $\vs_0$ 
is uniformly distributed on the unit circle.

We take $T=48$ and $d=20$ so that $\vx_{t}$ is a vector of
evaluations of a Gaussian function at 20 evenly spaced locations, 
and the peak of the Gaussian function follows the position of the particle 
$s_{t,1}$ for each $t$:
\[
x_{t,i} = \exp\left[-\frac{(\text{loc}(i)-s_{t,1})^2}{2\times0.3^2}\right],\\
\]
where $\text{loc}: \{1,\dots,d\}\to \sR$ maps pixel indices to a grid of evenly spaced points in the space of particle position. Thus, $\vx_t$, the observation at time $t$, contains information 
about $s_{t,1}$ but not $s_{t,2}$. A similar data generating process was 
used in \citep{WenliangSahani2019neurally}, inspired by \citet{johnson2016composing}.

We compare the marginal distribution
of the pixel values $x_{t,i}$ and joint distribution of the bump location ($\argmax_i {x_{t,i}}$)  between adjacent 
time steps. See \cref{fig:oscillation}.

\paragraph{Electroencephalography.} 
We obtain EEG dataset from \citep{eeg_data}
and take the recordings of all the 43 
subjects in the control group under the matching condition (S2). 
For each subject, we choose 75\% of the trials as 
training data and the remaining for evaluation, giving $2\,841$ training sequences and $969$ test sequences in total.
All data are subtracted by channel-wise mean, 
divided by three times the channel-wise standard deviation, 
and then passed through a $\tanh$ nonlinearity.
For COT-GAN, we train three variants corresponding to $\lambda$ 
being one of $\{1.0,0.1,0.01\}$, and $\eps=100.0$ for all OT-based methods.
Data and samples are shown
in \cref{fig:eeg_samples}.

\begin{figure}[t]
    \centering
    \includegraphics[width=\textwidth]{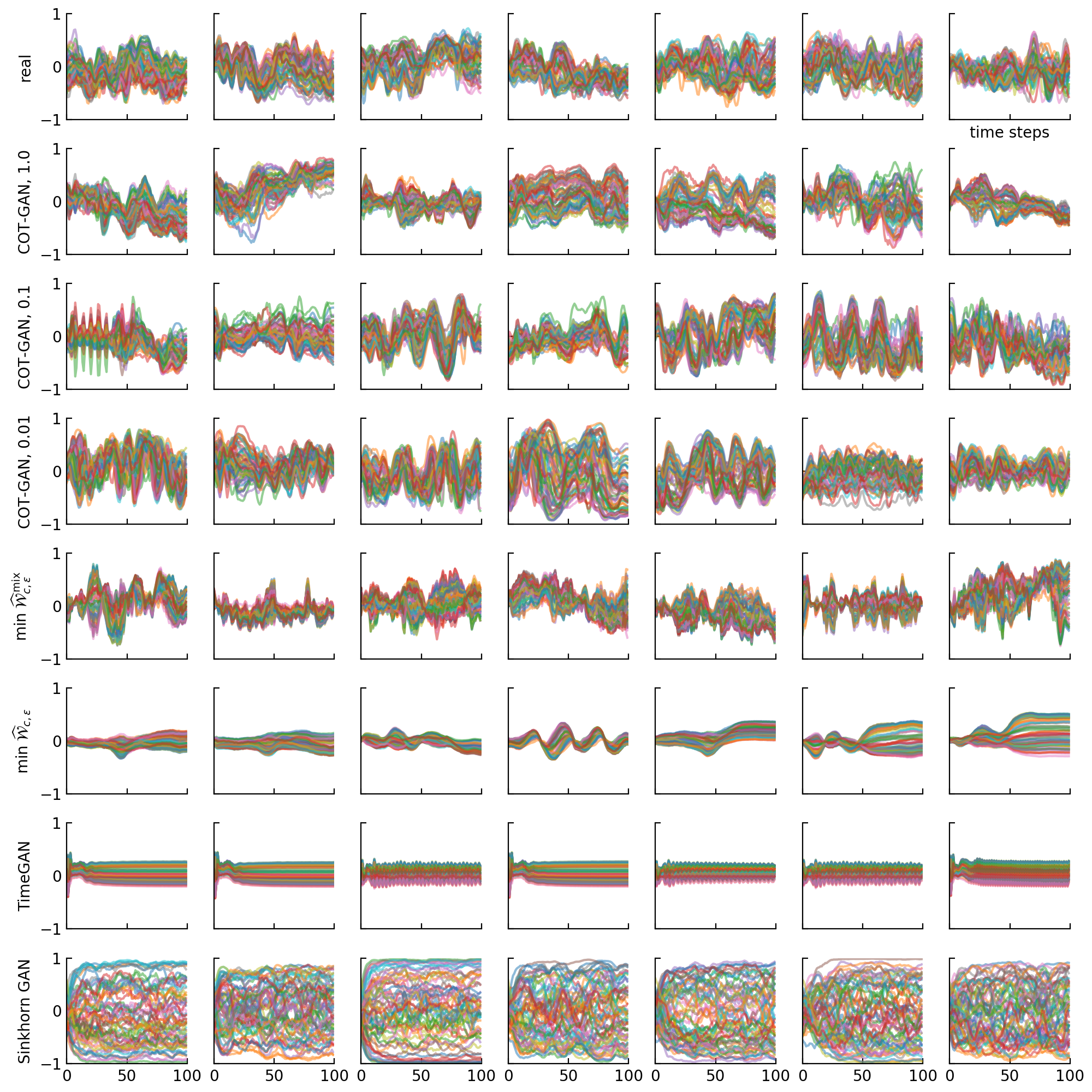}
    \caption{Data and samples obtained
    by different methods for EEG data, 
    the number after COT-GAN indicates the value of $\lambda$.}
    \label{fig:eeg_samples}
\end{figure}

\paragraph{Model and training parameters.}
The dimensionality of the latent state is 10 at each time step, and there is also a 
10-dimensional time-invariant latent state.
The generator common to COT-GAN, direct minimization and SinkhornGAN comprise a 
1-layer (synthetic) or 2-layer (EEG) 
LSTM networks, 
whose output at each time step 
is passed through two layers of fully connected $\operatorname{ReLU}$ networks.
We used Adam for updating $\theta$ and $\varphi$, 
with learning rate 0.001. 
Batch size is 32 for all methods
except for direct minimization of the mixed and original 
Sinkhorn divergence which is trained
 with batch size 64. 
These hyperparameters do not substantially affect the results.

The same discriminator architecture is used for both $h$ and $M$ in COT-GAN and 
the discriminator of the SinkhornGAN.
This network has two layers of 
1-D causal CNN with stride 1, filter length 5. 
Each layer has 32 (synthetic data) or 64 neurons (EEG) at each time step.
The activation is $\operatorname{ReLU}$ except at the output which is linear 
for autoregressive process, $\operatorname{sigmoid}$ for noisy 
oscillation, and $\tanh$ for EEG.

For COT-GAN, $\lambda=10.0$ and $\epsilon=10$ for synthetic datasets, and $\lambda \in\{0.01,0.1,1.0\}$ and $\epsilon=100.0$ 
for EEG. The choice of $\epsilon$ is made based on how fast it converges to a particular
threshold of the transport plan, and each iteration takes around 1 second on 
a 2.6GHz Xeon CPU.

\subsection{Videos datasets} \label{apx:gan_structure}

\subsubsection{Sprite animations}
\paragraph{Data pre-processing.}
The sprite sheets can be created and downloaded from
\footnote{Original dataset is available at \httpsurl{gaurav.munjal.us/Universal-LPC-Spritesheet-Character-Generator}/ \\ and \httpsurl{github.com/jrconway3/Universal-LPC-spritesheet}. To facilitate the use of large dataset in TensorFlow, we pre-shuffled all data used and wrote into tfrecord files. Links for download can be found on the Github repository.}. 
The data can be generated with various feature options for clothing, hairstyle and skin color, etc.\ Combining all feature options gives us 6352 characters in total.  Each character performs spellcast, walk, slash, shoot and hurt movements from different directions, making up to a total number of 21 actions.  As the number of frames $T$ ranges from 6 to 13, we pad all actions to have the same length $T=13$ by repeating previous movements in shorter sequences. We then crop the characters from sheets to be in the center of each frame, which gives a dimension of $64\times64\times4$ for each frame.  We decide to drop the 4th color channel (alpha channel) to be consistent 
with the input setting of baseline models.  Finally, the resulting dataset has 6352 data points consisting of sequences with 13 frames of dimensions $64\times64\times3$. 

\subsubsection{The Weizmann Action database}

\paragraph{Data pre-processing.}
The videos in this dataset consists of clips that have lengths from 2 to 7 seconds. Each second of the original videos contains 25 frames, each of which has dimension 144x180x3. To avoid the absence of objects at the beginning of the videos and to ensure an entire evolution of motions in each sequence, we skip the first 5 frames, then skip every 2 frames and collect 16 frames in a whole sequence as a result.  Due to limited access to hardware, we also downscale each frame to $64\times64\times3$. The training set used contains 89 data points with dimensions $16\times64\times64\times3$.

\paragraph{GAN architectures.}
We detail the GAN architectures used in the experiment of the Weizmann Action database in Table \ref{table:g_structure} and Table \ref{table:d_structure}. A latent variable $z$ of shape $5\times5$ per time step is sampled from a multivariate standard normal distribution and is then passed to a 2-layer LSTM to generate time-dependent features, followed by 4-layer deconvolutional neural network (DCONV) to map the features to frames. In order to connect two different types of networks,  we map the features using a feedforward (dense) layer and reshape them to the desired shape for DCNN.  In Table \ref{table:g_structure} and \ref{table:d_structure}, the DCONV layers have N filter size, K kernel size, S strides and P padding option.  We adopted batch-normalisation layers and the LeakyReLU activation function.
We have two networks to parameterize the process $h$ and $M$ as discriminator share the same structure, shown in Table \ref{table:d_structure}. 

\begin{table}
\centering
\caption{Generator architecture.}
 \begin{tabular}{||c c||}
 \hline
 Generator & Configuration  \\ 
 \hline\hline
 Input &  $z \sim \mathcal{N}(\mathbf{0}, \mathbf{I})$ \\ 
 \hline
 0 & LSTM(state size = 128), BN \\
 \hline
 1 & LSTM(state size = 256), BN \\
 \hline
 2 & Dense(8*8*512), BN, LeakyReLU \\
 \hline
 3 & reshape to 4D array of shape (m, 8, 8, 512) as input for DCONV \\ 
 \hline
 4 & DCONV(N512, K5, S1, P=SAME), BN, LeakyReLU \\ 
 \hline
 5 & DCONV(N256, K5, S2, P=SAME), BN, LeakyReLU \\ 
 \hline
 6 & DCONV(N128, K5, S2, P=SAME), BN, LeakyReLU \\ 
 \hline
 7 & DCONV(N3, K5, S2, P=SAME) \\
 \hline
\end{tabular}
\label{table:g_structure}
\end{table}

\begin{table}
\caption{Discriminator architecture.}
\centering
 \begin{tabular}{||c c||}
 \hline
 Discriminator & Configuration  \\ 
 \hline\hline
 Input & 64x64x3 \\
 \hline
 0 & CONV(N128, K5, S2, P=SAME), BN, LeakyReLU \\
 \hline
 1 & CONV(N256, K5, S2, P=SAME), BN, LeakyReLU \\
 \hline
 2 & CONV(N512, K5, S2, P=SAME), BN, LeakyReLU \\
 \hline
 3 & reshape to 3D array of shape (m, T, -1) as input for LSTM \\ 
 \hline
 4 & LSTM(state size = 512), BN \\ 
 \hline
 5 & LSTM(state size = 128) \\ 
 \hline
\end{tabular}
\label{table:d_structure}
\end{table}

We use a fixed length $T=16$ of LSTM. The state size in the last LSTM layer corresponds to the dimensions of $h_t$ and $M_t$, i.e., $j$ in (\ref{eq:cot_cost}). We also applied exponential decay to learning rate by $\eta_t = \eta_0 r^{s/c}$ where $\eta_0$ is the initial learning rate, $r$ is decay rate, $s$ is the current number of training steps and $c$ is the decaying frequency.  In our experiments, we set the initial learning rate to be $0.001$, decay rate $0.98$, and decaying frequency $500$. The batch size $m$ and time steps $T$ used are both 16.  We have $\lambda=0.01$, $\epsilon = 6.0$ and the Sinkhorn $L=100$ in this experiment.  We train COT-GAN on a single NVIDIA Tesla P100 GPU for 3 or 4 days.  Each iteration takes roughly 1.5 seconds.

\section{Sprites and human action results without cherry-picking}
\label{apx:humanactionresults}

In this section we show random samples of Sprites and human actions generated by COT-GAN, mixed Sinkhorn minimization, and MoCoGAN without cherry-picking.   
The background was static for both experiments.  In the Sprites experiments (see Figure \ref{apx:sprites}), the samples from mixed Sinkhorn minimization and COT-GAN are both of good quality, whereas those from MoCoGAN only capture a rough pattern in the frames and fail to show a smooth evolution of motions. 

\begin{figure}[h]
    \centering
    \includegraphics[width=0.48\textwidth]{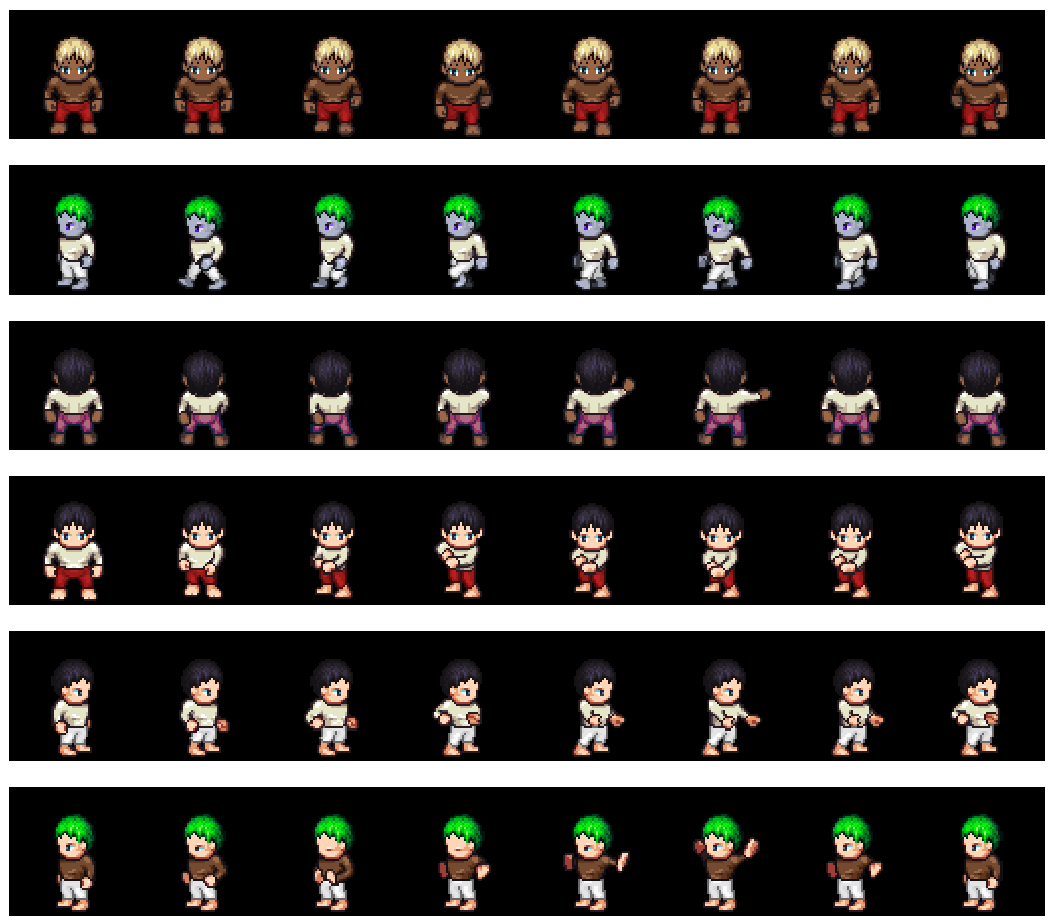}
    \includegraphics[width=0.48\textwidth]{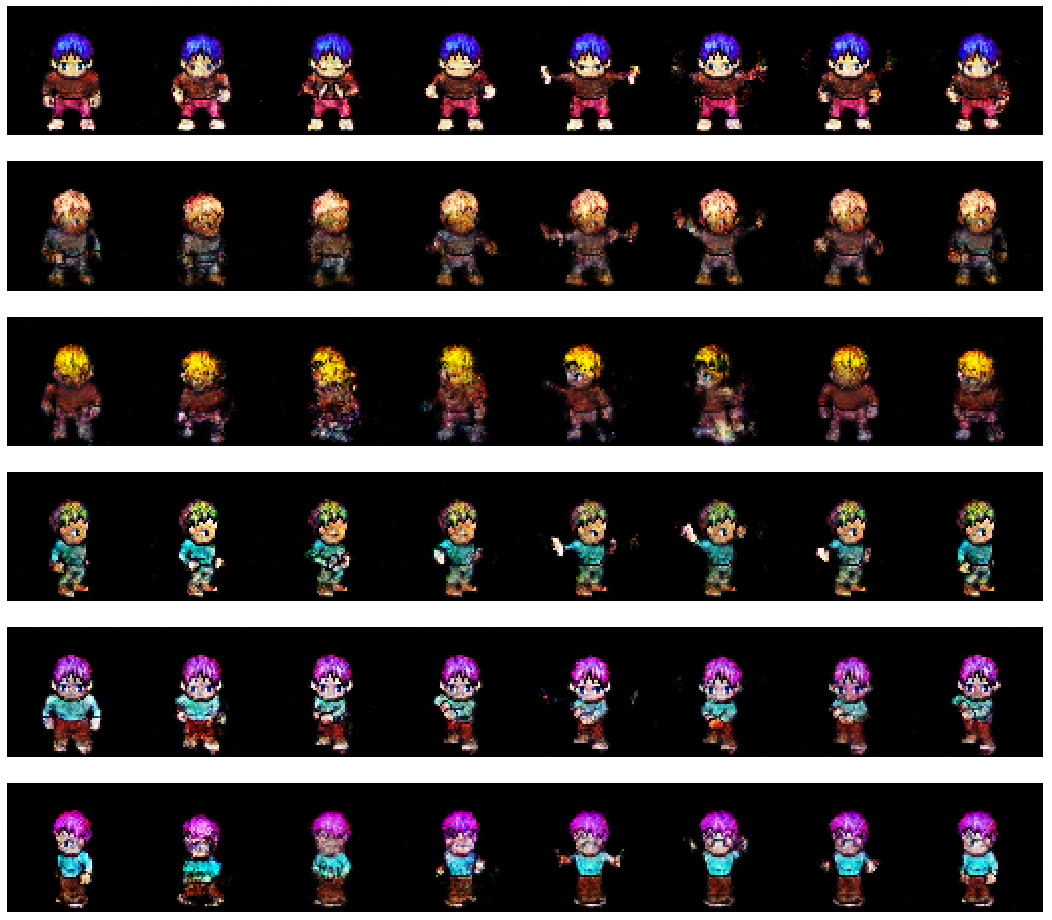}\\
     \vspace{1em}
    \includegraphics[width=0.48\textwidth]{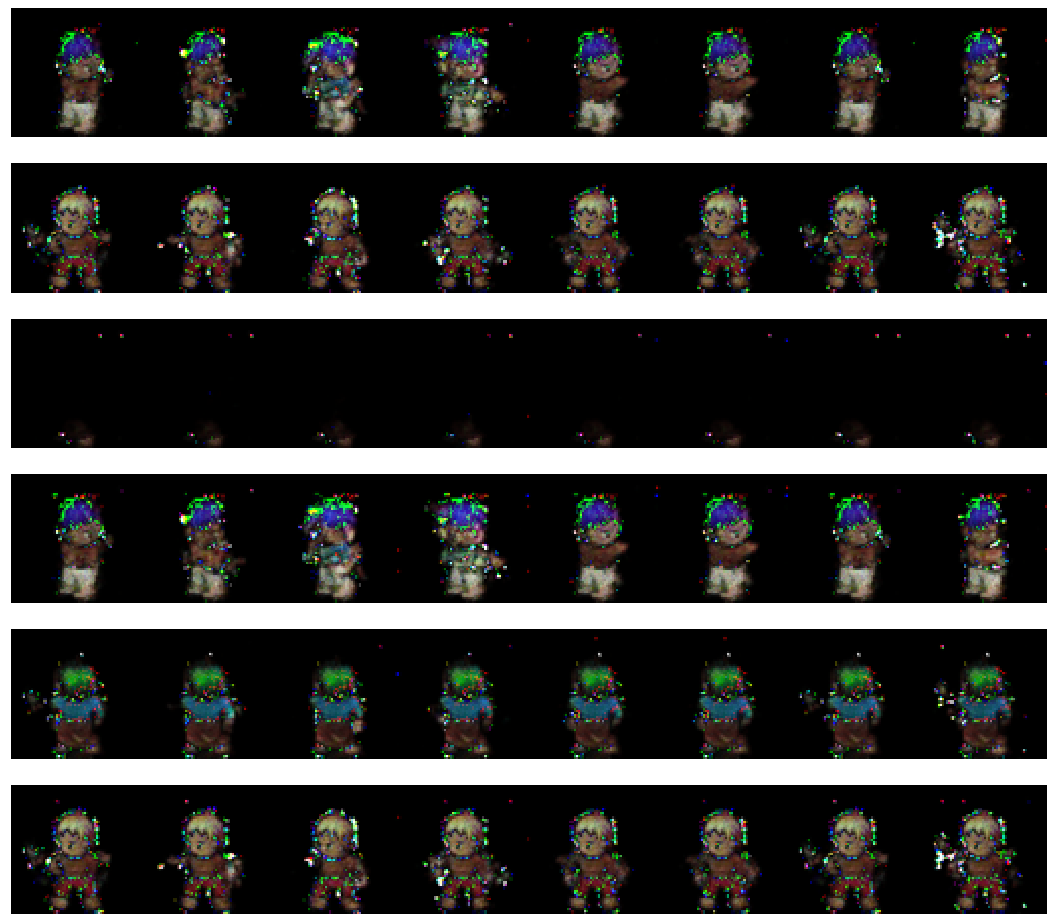}
    \includegraphics[width=0.48\textwidth]{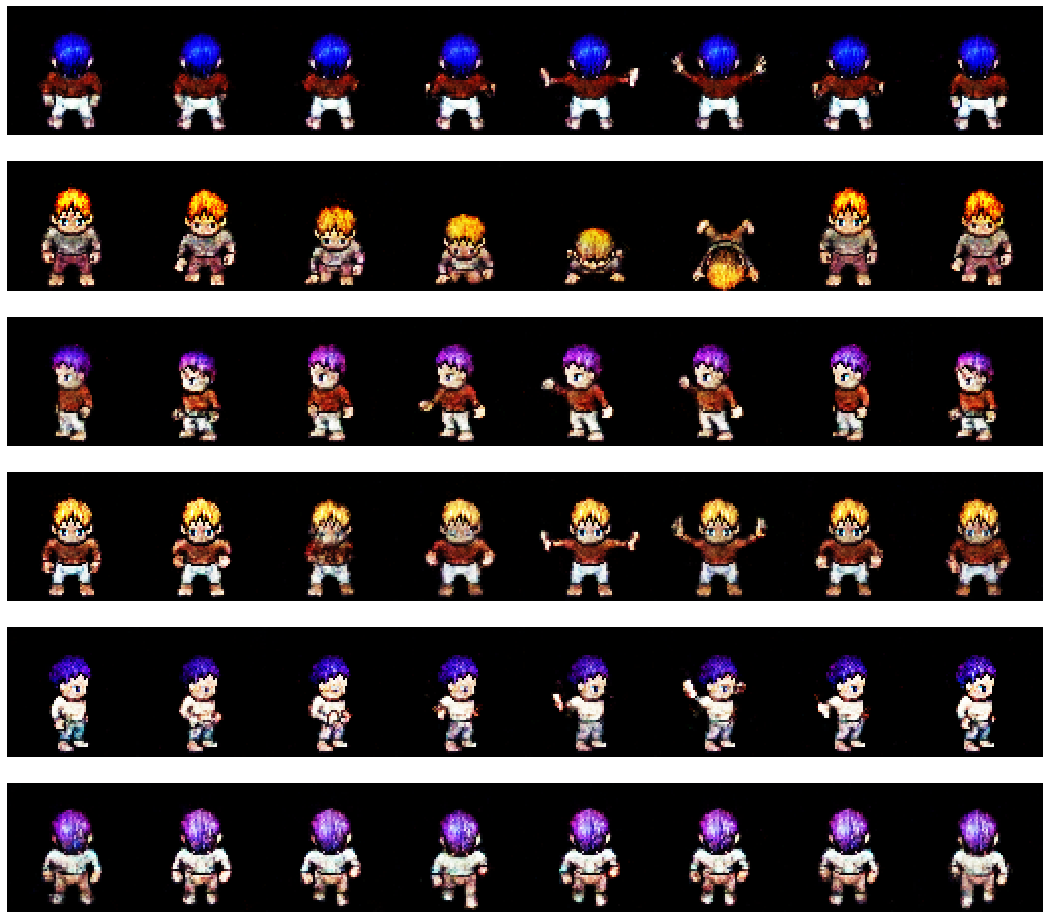}\\
    \caption{Random samples with no cherry picking from models trained on animated Sprites. Top row: real sequences on the left and mixed Sinkhorn minimization on the right; bottom row: MoCoGAN on the left and COT-GAN on the right.}
    \label{apx:sprites}
\end{figure}

In Figure \ref{apx:ha}, we show a comparison of real and generated samples for human action sequences. Noticeable artifacts of COT-GAN and mixed Sinkhorn minimization results include blurriness and even disappearance of the person in a sequence, which normally happens when the clothing of the person has a similar color as the background. MoCoGAN also suffers from this issue and, visually, there appears to be some degree of mode collapse. We used generators of similar capacity across all models and trained COT-GAN, mixed Sinkhorn minimization and MoCoGAN for 65000 iterations.  

\newpage
\begin{figure}[ht]
    \centering
    \includegraphics[width=0.48\textwidth]{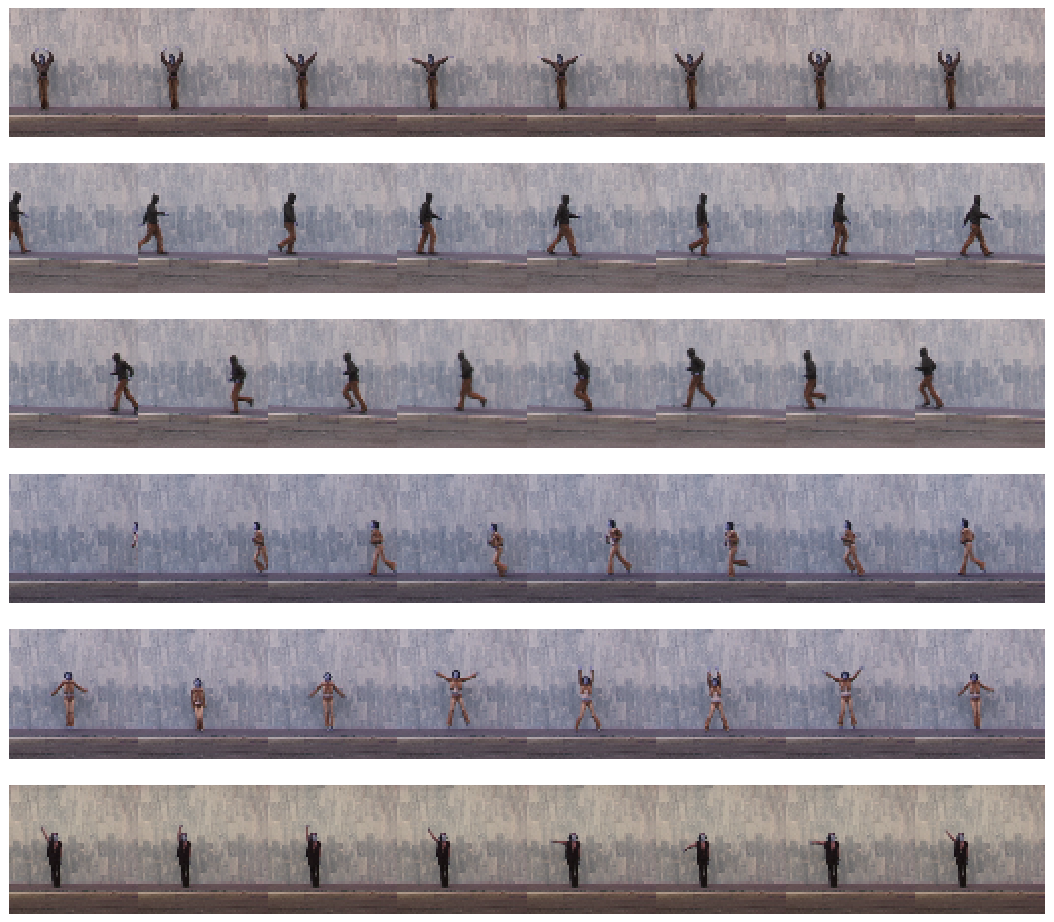}
    \includegraphics[width=0.48\textwidth]{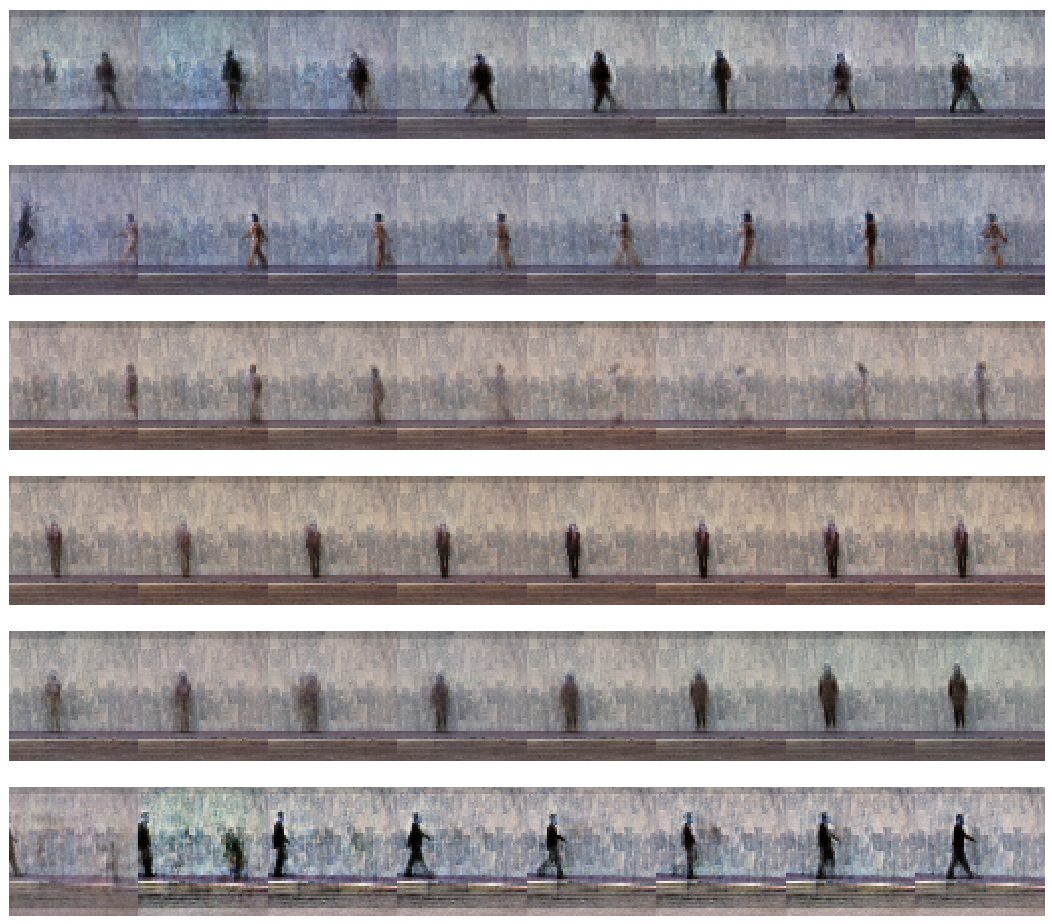} \\
     \vspace{1em}
    \includegraphics[width=0.48\textwidth]{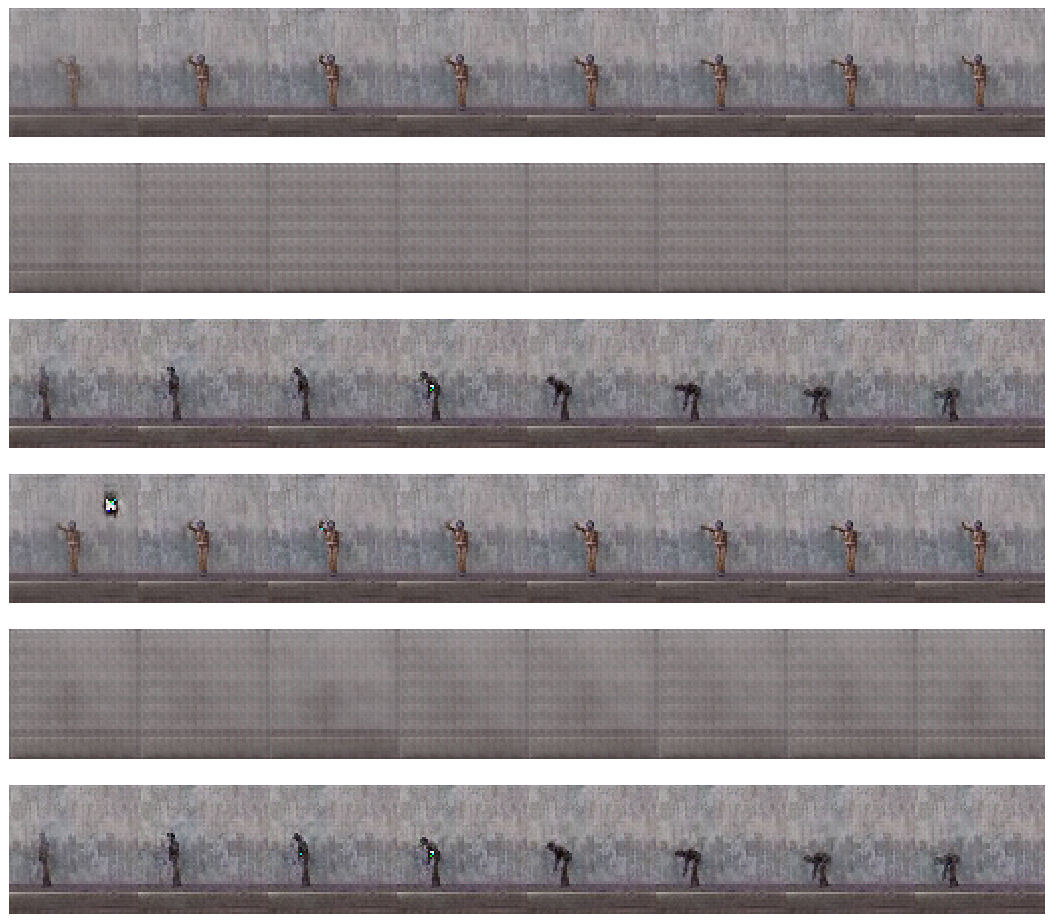}
    \includegraphics[width=0.48\textwidth]{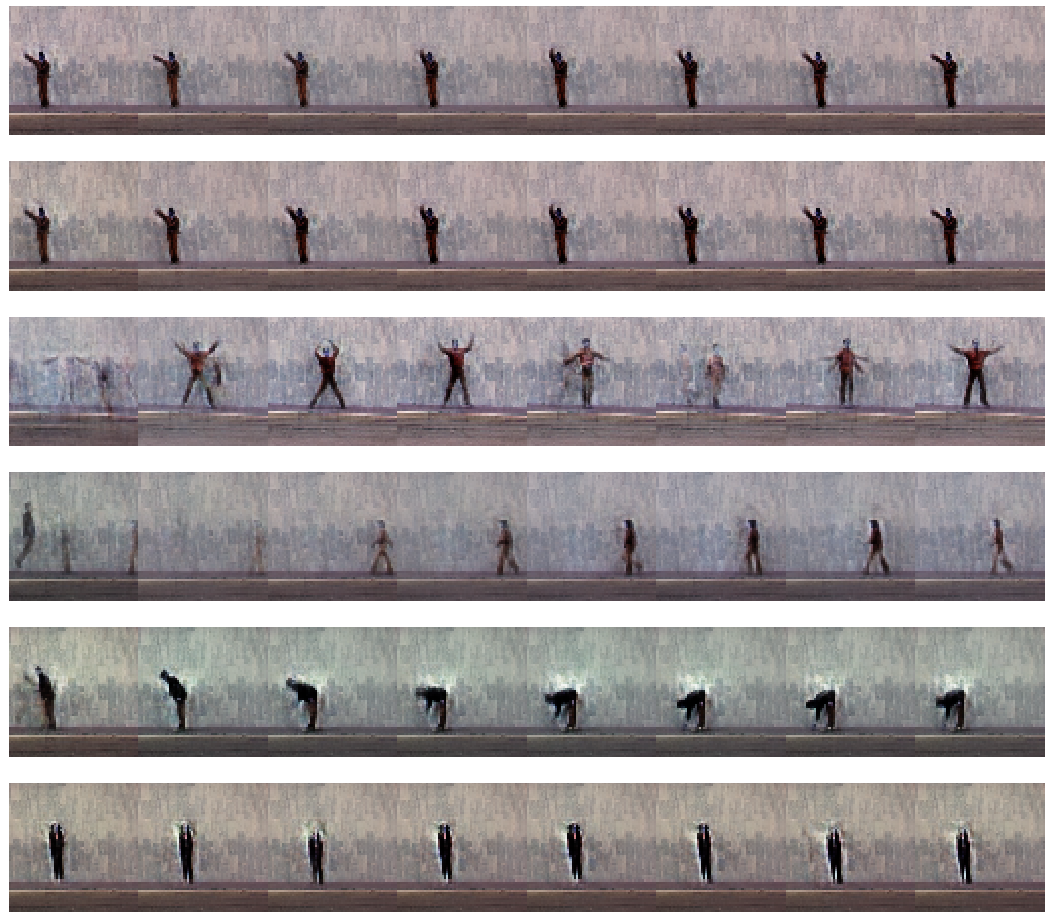}
    \caption{Random samples with no cherry picking from models trained on human actions. Top row: real sequences on the left and mixed Sinkhorn minimization on the right; bottom row: MoCoGAN on the left and COT-GAN on the right.}
    \label{apx:ha}
\end{figure}

\end{document}